\documentclass[accepted]{uai2024} % after acceptance, for a revised version;
                        
\usepackage[american]{babel}

\usepackage{natbib} % has a nice set of citation styles and commands
 \usepackage{mathtools}

\usepackage{booktabs} % commands to create good-looking tables
\usepackage{tikz} % nice language for creating drawings and diagrams
\usepackage{cleveref}
\usepackage{pifont}
\usepackage{subcaption}
\usepackage{amssymb}
\usepackage{amsthm}
\newtheorem{theorem}{Theorem}
\newtheorem{definition}{Definition}
\newtheorem{remark}{Remark}
\Crefname{lemma}{Lemma}{Lemmas}
\Crefname{opup}{Operator Update}{Operator Updates}

\crefname{equation}{Eq.}{Eqs.}
\crefname{figure}{Fig.}{Figs.}
\crefname{theorem}{Th.}{Ths.}
\crefname{algocf}{Algorithm}{Algorithms}

\usepackage{array}
\newcolumntype{L}[1]{>{\raggedright\let\newline\\\arraybackslash\hspace{0pt}}m{#1}}
\newcolumntype{C}[1]{>{\centering\let\newline\\\arraybackslash\hspace{0pt}}m{#1}}
\newcolumntype{R}[1]{>{\raggedleft\let\newline\\\arraybackslash\hspace{0pt}}m{#1}}

\usepackage[ruled,vlined]{algorithm2e}

\SetCommentSty{mycommfont}
\SetKwInput{KwDefine}{Define}
\SetKwInput{KwOptional}{Optional data}
\SetKwInput{KwRequire}{Require}
\SetKwInput{KwEnsure}{Ensure}

\usepackage{etoolbox}
\usepackage{thm-restate}

\newcommand*\squeezespaces[1]{% %% <- #1 is a number between 0 and 1
  \thickmuskip=\scalemuskip{\thickmuskip}{#1}%
  \medmuskip=\scalemuskip{\medmuskip}{#1}%
  \thinmuskip=\scalemuskip{\thinmuskip}{#1}%
  \nulldelimiterspace=#1\nulldelimiterspace
  \scriptspace=#1\scriptspace
}
\newcommand*\scalemuskip[2]{%
  \muexpr #1*\numexpr\dimexpr#2pt\relax\relax/65536\relax
} %% <- based on  https://tex.stackexchange.com/a/198966/156366

\makeatletter
\patchcmd{\@algocf@start}% <cmd>
  {-1.5em}% <search>
  {0pt}% <replace>
  {}{}% <success><failure>
\makeatother

\newcommand{\Pa}{\operatorname{Pa}}
\newcommand{\Ch}{\operatorname{Ch}} \newcommand{\Ne}{\operatorname{Ne}}
\newcommand{\Ad}{\operatorname{Ad}} \DeclareMathOperator*{\argmax}{arg\,max}

\newcommand{\indep}{\perp \!\!\! \perp}

\newcommand{\R}{\mathbb{R}}  
\newcommand{\curlystack}[1]{ \left\{\begin{matrix}#1\end{matrix}\right.}

\DeclareRobustCommand{\parhead}[1]{\noindent \textbf{#1} }

\newcommand{\data}{\bm{D}}
\usepackage[textsize=smaller]{todonotes}

\title{Extremely Greedy Equivalence Search}

\author[1]{\href{mailto:<aon2108@columbia.edu>?Subject=Your UAI 2024 paper}{Achille
Nazaret}{}}
\author[1,2]{David Blei}
\affil[1]{%
    Department of Computer Science\\
    Columbia University\\
    New York\\
    USA } 
\affil[2]{%
    Department of Statistics \\
    Columbia University\\
    New York\\
    USA } 
\begin{document}
\maketitle

\begin{abstract}
% !TEX root = ../main.tex

The goal of causal discovery is to learn a directed acyclic graph from data. 
One of the most well-known methods for this problem is Greedy Equivalence
Search (GES). 
GES searches for the graph by incrementally and greedily adding or removing edges
to maximize a model selection criterion. 
It has strong theoretical guarantees on infinite data but can fail in practice on
finite data. 
In this paper, we first identify some of the causes of GES's failure, finding that it 
can get blocked in local optima, especially in denser graphs. 
We then propose eXtremely Greedy Equivalent Search (XGES), which involves a new
heuristic to improve the search strategy of GES while retaining its theoretical
guarantees. 
In particular, XGES favors deleting edges early in the search over inserting
edges, which reduces the possibility of the search ending in local optima. 
A further contribution of this work is an efficient algorithmic
formulation of XGES (and GES). 
We benchmark XGES on simulated datasets with known ground truth. 
We find that XGES consistently outperforms GES in recovering the correct graphs, and 
it is 10 times faster.
XGES implementations in Python and C++ are available at \href{https://github.com/ANazaret/XGES}{https://github.com/ANazaret/XGES}.
\end{abstract}

\newcommand{\bm}[1]{\mathbf{#1}}

\newcommand{\cone}{\raisebox{.5pt}{\textcircled{\raisebox{-.9pt} {1}}}}
\newcommand{\ctwo}{\raisebox{.5pt}{\textcircled{\raisebox{-.9pt} {2}}}}

% !TEX root = ../main.tex

\section{Introduction}\label{sec:introduction} 

In the problem of causal discovery, we observe a multivariate data set
$x^{1:n}$, where each $x^i = (x^i_{1}, ..., x^i_{d})$. Our goal is to
learn a $d$-node directed graphical model for $p(x^i_{1}, \ldots, x^i_{d})$,
i.e., a factorization of the joint distribution. In practice, causal
discovery learns an equivalence class of graphs, called a
\textit{Markov equivalence class}, where each graph in the class
implies the same set of conditional independence statements. 
The goal is to find the class whose set of
independence statements exactly holds in the data.

The challenge to causal discovery is that the space of graphs on $d$
nodes is prohibitively large. To this end, researchers have explored a
number of ideas, including developing efficient tests for conditional 
independence \citep{spirtes2000causation,zhang2011kernel}, restricting
the space of graphs to a smaller class \citep{buhlmann2014cam,fang2023low}
, or searching efficiently the space of graphs.
One of the most theoretically
sound methods is \textit{greedy equivalence search} (GES) 
\citep{chickering2002optimal}. GES posits a proper scoring function 
for the graph (relative to the data) and then greedily optimizes it 
by inserting and deleting edges.

In the limit of large data, GES enjoys theoretical guarantees of
reaching the true graph. However, with finite data, GES can fail to
find the solution. 
In particular, its performance decreases for graphs with non-trivial
number of edges, e.g. more than two parents per node. And so we cannot apply GES to the
kinds of large-scale problems that we regularly encounter in machine
learning. 
To this end, researchers have proposed computationally efficient 
approximations \citep{ramsey2017million} and continuous relaxations with gradient-based 
optimization \citep{zheng2018dags,brouillard2020differentiable}. 
These methods can handle more variables and denser graphs, but they do not enjoy the same guarantees.

In this paper, we improve on GES in two ways. First, we empirically
examine the failure modes of GES and then use this analysis to propose
better heuristics to explore the space of DAGs. Second, we develop
superefficient algorithms for implementing the low-level graph
operations that GES requires. Put together, these innovations describe
extreme GES (XGES), a new algorithm for causal discovery.

XGES is more reliable and scalable than GES, and without sacrificing
its important theoretical guarantees. While GES's performance degrades as 
the density of edges increases, XGES's performance remains stable.
We study XGES on a battery of simulations. We find that XGES outperforms GES and its variants in 
all scenarios, achieving significantly better accuracy and faster runtimes.
 
% !TEX root = ../main.tex

\parhead{Related Work.}
Causal discovery encompasses a wide range of methods \citep{glymour2019review}.
Here, we focus on score-based methods, which posit a proper scoring rule, and proceed
to find the sets of graphs that maximize it. The Greedy Equivalence Search (GES) algorithm maximizes it using a
greedy search strategy \citep{chickering2002optimal}. Extensions and variants of GES include OPS \citep{chickering2002optimal}, GIES
\citep{hauser2012characterization}, GDS-EEV \citep{peters2014identifiability},  and ARGES \citep{nandy2018high}.

Fast-GES (fGES) is a more efficient implementation of GES\citep{ramsey2017million}. But
we find that it does not reproduce exactly GES's search strategy and hurts performance
(see \Cref{sec:experiments}). Selective GES (SGES) guarantees a polynomial
worst-case complexity but has limited speed improvement in practice
\citep{chickering2015selective}.

Other works improve GES by randomly perturbing the search
\citep{alonso2018use,liu2023improving}. However, they are computationally
expensive. 

More recently, \textit{differentiable causal discovery} methods have been proposed to
maximize the score using gradient-based methods
\citep{zheng2018dags,brouillard2020differentiable,nazaret2023stable}. These methods can model causal
relations using neural networks, but unlike GES and XGES (proposed here), their optimization procedures have no
theoretical guarantees of converging to the true graph. 

% !TEX root = ../main.tex

\section{Causal Discovery and GES}\label{sec:background}  

We first review the causal discovery problem and the necessary details of the greedy
equivalence search (GES) method.

\subsection{Causal Graphical Models }
Causal discovery aims to identify cause-and-effect relationships between random
variables $\{X_1, ..., X_d\}$. We reason about causal relationships using causal
graphical models (CGM). A CGM has two components:
\begin{enumerate}
\item a directed acyclic graph (DAG), $G^* = (V,E)$, where a node $j \in V$ represents
variable $X_j$ and an edge $(j, k) \in E$ denotes a direct causal link from $X_j$ to
$X_k$,
\item conditional distributions $p(X_j  \mid X_{\Pa_j^{G^*}})$, defining the
distribution of $X_j$ given its causal parents $X_{\Pa^{G^*}_j}$.
\end{enumerate}
The joint distribution of the variables $X_1,..., X_d$ writes:
\begin{equation}\label{eq:factor_joint}
p^*(X) = \prod\limits_{j \in V} p^*(X_j \mid X_{\Pa^{G^*}_j}).
\end{equation}
The goal of causal discovery is to recover the graph $G^*$ from the joint distribution
$p^*$ or from samples drawn from $p^*$. 

However, multiple CGMs with different graphs can generate the same $p^*\!.$ Two
important concepts address this difficulty: faithfulness and Markov equivalence
\citep{spirtes2000causation}.

\parhead{Faithfulness.} In \Cref{eq:factor_joint}, $G^*$ induces a factorization of
$p^*$, which, in turn, induces independencies between variables: each $X_j$ is
independent of its non-descendants given its parents $X_{\Pa^{G^*}_j}$
\citep{pearl1988probabilistic}. Reciprocally, we say that a distribution $p^*$ and a
graph $G^*$ are \textit{faithful} if all the independencies in $p^*$ are exactly those
implied by $G^*$ and no more. \looseness=-1

%(Notice it is the lack of edge in $G^*$ that imposes independencies in $p^*$.)
For example, $H=(\{1,2\}, \{1\rightarrow 2\})$ and $q=q(X_1)q(X_2)$ form a
valid CGM. But the independence $X_1 \indep X_2$ present in $q$ is not suggested by $H$.
Rather, $q$ is faithful to $H'=(\{1,2\}, \varnothing)$, which has no superfluous edges
like $1 \rightarrow 2$.

Limiting the search to faithful graphs reduces the possible CGMs that could have
generated $p^*$. But this is not enough.

\parhead{Markov Equivalence.} 
Two distinct graphs can both be faithful to $p^*$ if they induce the same set of
independencies on $p^*$. For example, $A\rightarrow B \rightarrow C$ and $A \leftarrow B
\leftarrow C$ impose the same set of independencies,  $\{A \indep C \mid B\}$. 

Graphs inducing the same independencies are called \textit{Markov equivalent}, they form
\textit{Markov equivalence classes} (MEC).
Since $G^*$ is identifiable only up to Markov equivalence, the task of causal discovery
becomes finding the MEC of $G^*$. 

\begin{remark}
    With more assumptions (e.g. about the form of $p^*$) or special data (e.g.,
    interventions), other causal discovery methods focus on identifying the possible
    $G^*$ beyond Markov equivalence. See \citet{glymour2019review} for an excellent
    review. We focus on Markov equivalence classes.
\end{remark}

\subsection{Score-based Causal Discovery}
In this work, we assume to have $n$ iid samples, denoted $\data = \{(x_1^{i}, ...,
x_d^{i})\}_{i=1}^n$, from a distribution $p^*$ that is faithful to some $G^*$. We aim to
recover the MEC of $G^*$ from $\data$.

Greedy Equivalence Search (GES) searches for the MEC of $G^*$ among all the possible
MECs. It does so by searching for the MEC whose DAGs maximize a specific score function.
It is a particular case of score-based causal discovery.

{Score-based methods} assign a score $S(G; \data)$ to every possible DAG $G$ given the
data $\data$. The score function $S$ is designed to be maximized by the true graph
$G^*$. This turns causal discovery into an optimization problem. 
\begin{equation}
    G^* = \argmax_{G} S(G; \data)
\end{equation}
Some scores have properties that are important for GES.

\begin{definition}
    A score $S$ is \emph{score equivalent} if it assigns the same score to all the
    graphs in the same MEC.
\end{definition}

A score equivalent $S$ enables defining the score of a MEC as the score of any of its
constituent graphs. 

\begin{definition}[Local Consistency, \citep{chickering2002optimal}]
Let $\data$ contain $n$ iid samples from some $p^*$. Let $G$ be any DAG and $G'$ be a
different DAG obtained by adding the edge $i \rightarrow j$ to $G$. A score $S$ is
\emph{locally consistent} if both hold:
\begin{enumerate}
    \item $X_i \not\hspace*{-0.5mm}\indep_{p^*} X_j \mid X_{\Pa_j^G} \Rightarrow S(G';
    \data)>S(G; \data)$,
\item $X_i \indep_{p^*} X_j \mid X_{\Pa_j^G} \Rightarrow S(G'; \data)<S(G; \data)$.
\end{enumerate}
The independence statements are with respect to $p^*$.

\end{definition}

A locally consistent score increases when we add an edge that captures a dependency not
yet represented in the graph. It decreases when we add an edge that does not capture any
new dependency.

\subsection{Greedy Equivalence Search}
\citet{chickering2002optimal} introduced Greedy Equivalence Search (GES). It is a
score-based method that is guaranteed to return the MEC of $G^*$ whenever the score is
locally consistent. The main characteristic of GES is to navigate the space of MECs.

GES begins with the MEC of the empty DAG and iteratively modifies it to improve the
score. At each step, only a few modifications are allowed. These
modifications are of three types: insertions, deletions, and reversals. 
\begin{figure}
    \centering
    \includegraphics[width=\linewidth]{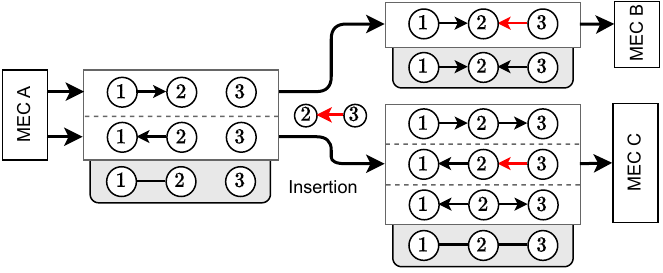}
    \caption{Illustration of insertions from a MEC A to MECs B or C: (i) choose a DAG in
    A, (ii) insert the edge $2\leftarrow3$ to obtain another DAG (iii) consider its MEC.
    Each MEC has all its DAGs on a white plate, and its canonical PDAG on a gray
    plate (see \Cref{sec:implementation} for a definition).}
    \label{fig:mec}
\end{figure}

\parhead{MEC modifications.}An \textit{insertion} on MEC $M$ selects a DAG $G$ in $M$,
adds an edge $x \rightarrow y$ to $G$ to obtain a different DAG $G'$ and replaces $M$
with the MEC of $G'$ (see \Cref{fig:mec}).

A \textit{deletion} on MEC $M$ selects a DAG $G$ in $M$, removes an edge
$x \rightarrow y$ from $G$ to obtain a different DAG $G'$ and replaces $M$ with the MEC
of $G'$.

A \textit{reversal} on MEC $M$ selects a DAG $G$ in $M$, reverses an edge
$x \rightarrow y$ to $x \leftarrow y$ in $G$ to obtain a different DAG $G'$ and replaces
$M$ with the MEC of $G'$.

GES applies these modifications in three separate phases.

\parhead{Phase 1: Insert.} First, GES finds all possible insertions, 
applies the one leading to the largest score increase, and repeats until no insertion
increases the score.

\parhead{Phase 2: Delete.} Then, GES finds all possible deletions, applies the one
leading to the largest score increase, and repeats until no deletion increases the
score.

The MEC obtained at the end of phase 2 is exactly the MEC of $G^*$ if the score is
locally consistent \citep{chickering2002optimal}

\parhead{Phase 3: Reverse.} In theory, phases 1 and 2 are sufficient to recover the
MEC of $G^*$. Yet, \citet{hauser2012characterization} showed that adding a third
phase with reversals can improve the search in practice with finite data (where the score might not be
locally consistent). In this phase, GES finds all possible reversals for the MEC,
applies the one that increases the score most, and repeats until none do.

The pseudocode of GES is given in \Cref{alg:ges-vanilla}.

\begin{algorithm}[t]
    \caption{Greedy Equivalence Search (GES)}\label{alg:ges-vanilla}  
    \DontPrintSemicolon
    
    \KwIn{Data $\data \in \R^{n\times d}$, score function $S$}  
    \KwDefine{$\delta_{\data,M}(O) = S(\text{Apply}(O,M) ; \data) - S(M; \data)$}  
    \KwOut{MEC of $G^*$}  
    $M \leftarrow \{([d], \varnothing)\}$ \tcp*{Empty graph's MEC}  
    $\mathcal{I} \leftarrow$ get all insertions valid for $M$ \;  
    \While{$|\mathcal{I}| > 0$}{  
        $O^* \leftarrow \argmax_{I \in \mathcal{I}}\{\delta_{\data,M}(I)\}$ \tcp*{Get
        best insertion}  
        \lIf{$\delta_{\data,M}(O^*) \leq 0$}{\textbf{break}}  
        $M \leftarrow$ Apply($O^*, M$) \tcp*{Apply best insertion}  
        $\mathcal{I} \leftarrow$ get all insertions valid for $M$ \;  
    }  
    $\mathcal{D} \leftarrow$ get all deletions valid for $M$ \;  
    \While{$|\mathcal{D}| > 0$}{  
        $O^* \leftarrow \argmax_{D \in \mathcal{D}} \{\delta_{\data,M}(D)\}$ \tcp*{Get
        best deletion}  
        \lIf{$\delta_{\data,M}(O^*) \leq 0$}{\textbf{break}}  
        $M \leftarrow$ Apply($O^*, M$) \tcp*{Apply best deletion}  
        $\mathcal{D} \leftarrow$ get all deletions valid for $M$ \;  
    }  
    \tcc{(Optional) 3rd phase like above but with reversals}  
    \Return $M$ \;  
\end{algorithm}

\parhead{Correctness.} GES's correctness relies on two properties:
    (i) the greedy scheme will reach the global maximum of any locally consistent
    score \citep[Lemma 10]{chickering2002optimal}, and 
    (ii) the true graph $G^*$ and its MEC are the unique global maximizers of any
    locally consistent score \citep[Proposition 8]{chickering2002optimal}.

With these two properties, GES is guaranteed to recover the MEC of $G^*$ when the score
is locally consistent.
\label{sec:greedy_search}

\parhead{The BIC score.} A score commonly used by GES is the Bayesian Information Criterion
(BIC) \citep{schwarz1978estimating}. Given a model class
$\mathcal{M}_G$ for each $G$, and our data $\data$ with its $n$ samples, the BIC defines a score:
\begin{equation}
    S(G; \data) = \log p_{\hat\theta}(\data ; G) - \frac{\alpha}{2} \log n \cdot |\hat\theta|,
\end{equation}
where
$\alpha$ is a hyperparameter, $\hat \theta$ is the likelihood maximizer over
$\mathcal{M}_G$, and the number of parameters $|\hat\theta|$ is usually the number of edges in $G$. The BIC is a model selection criterion trading off log-likelihood and model complexity. Models with higher BIC are preferred. The original BIC has $\alpha=1$.

\parhead{Gaussian Linear Models. } A common model class used for continuous data with the BIC is the class of Gaussian linear models, where given a graph $G$, each variable is Gaussian with a linear
conditional mean and a specific variance:
\begin{equation}\textstyle
    p_{\theta}(x_j \mid x_{\Pa_j^G}) \sim \mathcal{N}\Bigl(\sum_{k \in \Pa_j^G} 
    \theta_{jk} x_k + \theta_{j0}, \theta_{j(d+1)}^2\Bigr).
\end{equation}
For Gaussian linear models, the BIC is score equivalent. It can also be locally
consistent under some conditions.

\begin{theorem}[Local Consistency of BIC \citep{haughton1988choice,chickering2002optimal}]
    For $\alpha>0$, the BIC for Gaussian linear models is locally consistent once $n$
    is large enough.
    \label{th:local_consistency}
\end{theorem}

% More generally, the local consistency theorem also applies to most model classes of the exponential family (e.g. the multinomial distribution for discrete data).
Hence, for Gaussian linear models, GES is guaranteed to recover the MEC of $G^*$ with
infinite data. 
In practice, however, data is finite and GES can return incorrect MECs. 

\parhead{Example of Failure.} In \Cref{fig:simulation_main}, we report the performance
of GES (orange) on simulated data, along with its variants and the proposed XGES. We
simulate CGMs $(G^*, p^*)$ for $d \in\{25,50\}$ variables, $\rho d \in \{2d,3d,4d\}$
edges ($\rho$ is an edge density parameter) and draw $n=10,000$ samples from $p^*$. The
simulation is detailed in \Cref{sec:experiments:evaluation_setup}. We compare the
methods' results to the true graph $G^*$ using the structural Hamming distance for MECs
(SHD), which counts the number of different edges between graphs of two MECs (see
\Cref{sec:experiments:evaluation_setup}). \Cref{fig:simulation_main} shows that GES can fail (SHD $> 0$),
especially in denser graphs.

In addition, we confirm that including the third phase of reversals in GES improves its
performance (GES-r, in green).

% figure: simulation_main
\begin{figure}[t]
    \centering
    \includegraphics[width=1\linewidth]{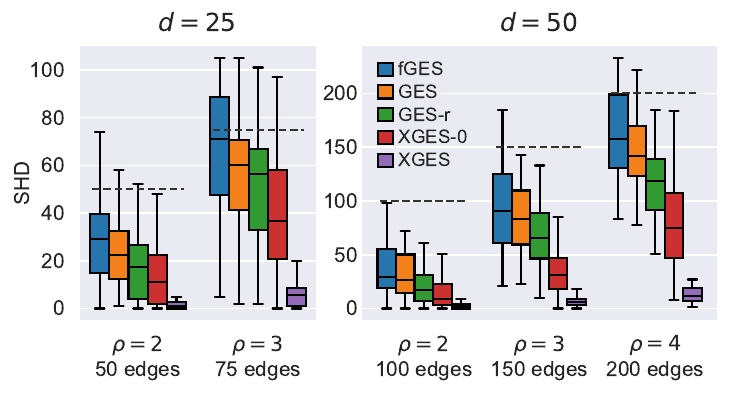}
    \caption{Performance comparison of GES and XGES variants, measured with SHD for
    different edge densities $\rho$. XGES heuristics outperform GES and its variants in
    all scenarios. The dashed lines indicate the number of edges of the true graph. Each boxplot
    is computed over 30 seeds.\looseness=-1}
    \label{fig:simulation_main}
\end{figure}

% !TEX root = ../main.tex

\section{Extremely Greedy Equivalence Search}\label{sec:xges}

In this section, we first investigate why GES can fail and then
propose simple solutions to mitigate failure.

\subsection{Scenarios of GES failure}
\label{sec:ges_failure}
As reviewed in \Cref{sec:greedy_search}, GES's correctness relies on two conditions: (i)
the ability of its greedy search to find the score's global maximizer, and (ii) the true
graph's MEC being the score's global maximizer (and the only one).

These two properties might not hold if the data cannot render $S$
locally consistent. We investigate GES failure.
Is the issue that GES cannot reach the global maximizer with greedy search? If so,
changing the search heuristic could help bypass local maxima. Or is it that GES
effectively finds the global maximizer, but this maximizer is not the true graph?
In this case, designing a new score might help. To check these hypotheses, we conduct
empirical experiments.
\begin{figure}
    \centering
    \includegraphics[width=\linewidth]{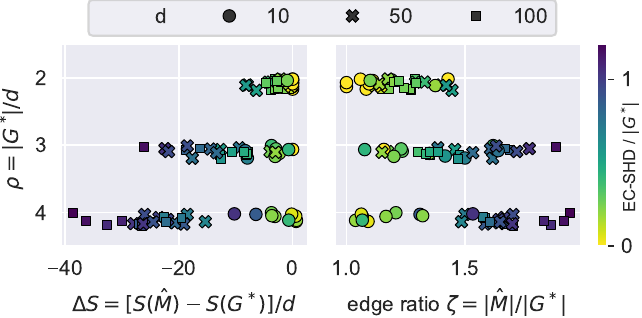}
    \caption{
    Empirical study of GES failure, on 90 simulated datasets with varying variables $d$ and graph densities $\rho$. (left) Differences in BIC
    between GES and ground-truth are negative. GES does not find the score's
    global maximum. (right) Ratios of GES-edges to true edges exceed 1. GES returns many more edges than the true graph.
}
    \label{fig:ges_fail}
\end{figure}

\parhead{Greedy Optimization Fails.}
We apply GES on the data simulated for \Cref{fig:simulation_main}, this time including
$d \in \{10,50,100\}$ variables. We obtain a MEC $\hat M$ that we compare to $G^*$ with:
\begin{equation}
    \Delta S = \frac{S(\hat M; \data) - S(G^*; \data)}{d}.
\end{equation}
A negative $\Delta S$ indicates that GES failed to identify a global maximizer and that
$G^*$ still scores higher than $\hat M$ (note that it does not imply $G^*$ is the global maximizer). A positive $\Delta S$ demonstrates
that $G^*$ is no longer the global maximizer and that GES legitimately identified a
MEC with high score.

In \Cref{fig:ges_fail} (left), we consistently find $\Delta < 0$ across different
numbers of variables $d$ and edge densities $\rho$. This is evidence that GES fails to
reach the global maximum and that the true graph $G^*$ still has a better score than
$\hat M$.

We measure next the ratio $\zeta = |\hat M| /|G^*|$ between the number of edges found by
GES and the number of true edges. \Cref{fig:ges_fail} (right) shows that GES
over-inserts edges i.e. $\zeta > 1$, especially with dense graphs. The idea behind GES
is to over-insert edges in the first phase and then delete them in the second phase.
However, we hypothesize that over-inserting may lead GES into local maxima before the
second phase can correct it.

Following these two observations, we focus on ensuring that GES reaches the global
maximizer. To do so, we design novel search heuristics aimed at preventing
over-insertion.

\subsection{The Heuristic XGES-0}\label{sec:xges-0}

GES considers insertions, deletions, and optionally reversals in three separate phases.
Rather, we consider all operations simultaneously, where deletions, insertions, and
reversals can interleave in any order. And when both insertions and deletions can 
increase the score, we prioritize deletions. 

\parhead{Heuristic XGES-0.} At each step, identify all the valid insertions,
deletions, and reversals. If some deletes would increase the score, apply the best one.
Otherwise, if some reversals would increase the score, apply the best one. Otherwise,
apply the best insert. Repeat until no deletions, reversals, or insertions can increase
the score.

\begin{algorithm}[t]
    \DontPrintSemicolon  
     % align line on the right
    \KwIn{Data $\data \in \mathbb{R}^{n\times d}$, score $S$.}
    \KwDefine{$\delta_{\data,M}(O) = S(\text{Apply}(O,M) ; \data) - S(M; \data)$}  
    \KwOut{MEC of $G^*$}  
    $M \leftarrow \{([d], \varnothing)\}$\;  
    $\mathcal{I},\mathcal{D},\mathcal{R} \leftarrow$ all insertions, deletions, reversals
    valid for $M$\;  
    \While{$|\mathcal{I}| + |\mathcal{D}| + |\mathcal{R}| > 0$}{  
        \uIf{$|\mathcal{D}| > 0$ \textnormal{\textbf{and}} $\max_{D \in \mathcal{D}}
        \{\delta_{\data,M}(D)\} \geq 0 $}{  
            $O^* \leftarrow \argmax_{D \in \mathcal{D}} \{\delta_{\data,M}(D)\}$ \;  
        }  
        \uElseIf{$|\mathcal{R}| > 0$ \textnormal{\textbf{and}} $\max_{R \in \mathcal{R}}
        \{\delta_{\data,M}(R)\} > 0 $}{  
            $O^* \leftarrow \argmax_{R \in \mathcal{R}} \{\delta_{\data,M}(R)\}$\;  
        }  
        \uElseIf{$|\mathcal{I}| > 0$ \textnormal{\textbf{and}} $\max_{I \in \mathcal{I}}
        \{\delta_{\data,M}(I)\} > 0 $}{  
            $O^* \leftarrow \argmax_{I \in \mathcal{I}} \{\delta_{\data,M}(I)\}$\;  
        }  
        \uElse{  
            \textbf{break} \tcp*{No more operations available}  
        }  
        $M \leftarrow$ Apply($O^*, M$) \;
        $\mathcal{I},\!\mathcal{D},\!\mathcal{R} \leftarrow\!$ all insertions, deletions, reversals
        valid for \!$M$\;  
    }  
    \Return $M$\;  
    \caption{XGES-0.   
    }  
    \label{alg:xges-0}  
\end{algorithm}

We call this heuristic XGES-0 for eXtremely Greedy Equivalence Search and we detail it in
\Cref{alg:xges-0}. XGES-0 retains the same theoretical correctness as GES.

\begin{restatable}[]{theorem}{thXgesZero}
    For any locally consistent score $S$, the MEC $\hat M$ returned by XGES-0 contains
    the true graph $G^*$.
    \label{thm:xges0}
\end{restatable}

The proof leverages the same theorems as those used to prove GES's correctness. It is
provided in \Cref{appendix:sec:theoretical_guarantees_xges0}.

In \Cref{fig:simulation_main}, we find empirically that XGES-0 (red) obtains MECs
with better scores and closer to $G^*$ than GES. Early deletions effectively reduce
encounters with local maxima.

\begin{remark}
    Alongside GES, \citet{chickering2002optimal} proposed a variant called OPS that
    also considered insertions and deletions simultaneously. But OPS did not prioritize
    deletions over insertions, resulting in no improvements to GES in practice. 
    We provide more details in \Cref{appendix:sec:empirical:ops}.
    %Mathematically, deleting an edge can only increase the BIC score by at most $\frac{\alpha}{2} \log n$ which is usually smaller than the increase from inserting an edge. Hence even if OPS considered deletions, it would still mostly insert edges first and encounter the same local maxima as GES.
\end{remark}

\subsection{The Heuristic XGES}

Building upon XGES-0, we introduce the heuristic XGES. XGES complements XGES-0. It
repeatedly uses XGES-0, each time deleting an edge that causes a local maximum.

\parhead{Heuristic XGES.} XGES begins by applying XGES-0 until no operations can
increase the score. Then, it enumerates all valid deletions, all of which will
decrease the score. For each deletion, XGES copies the MEC, applies the deletion, and resumes XGES-0 on the
copy, but without ever reinserting the edge removed by the deletion.
If the final score is worse than the original MEC, the copy is discarded, and the search
continues with the next deletion. If the final score is better, the copy becomes the new
MEC. XGES then restarts with the new MEC and all its new deletions. XGES stops once all
deletions of a MEC have been tried. We provide the pseudocode of XGES in \Cref{alg:xges}
in \Cref{appendix:sec:theoretical_guarantees_xges}.

\parhead{Intuition.} The XGES heuristic aims to remove incorrect edges that were
inserted early in the search and might be causing local maxima, preventing their
deletion. By forcefully deleting these edges, XGES can get around the local maximum and
discover better graphs.

\begin{restatable}[]{theorem}{thXges}
    For any score $S$, XGES returns a MEC $\hat M$ with a higher or equal score than
    XGES-0. If $S$ is locally consistent, then $\hat M$ contains the true graph $G^*$. 
    \label{thm:xges}
\end{restatable}

The proof follows from the design of XGES and by \Cref{thm:xges0}. 

\Cref{fig:simulation_main} illustrate the performance of XGES (purple), showing that it
significantly improves over all other GES variants.  XGES enables
non-trivial causal discovery in denser graphs ($\rho\geq2$) with many variables
($d\geq50$). However, XGES is computationally more expensive than
XGES-0. To alleviate this, we develop an efficient implementation for it.

% !TEX root = ../main.tex

\section{Efficient Algorithm}\label{sec:implementation} GES, if naively implemented, is
slow. In this section, we develop new ways of implementing its details to significantly speed it up. As we will see in the empirical studies, these details are crucial to scaling up XGES to large, dense graphs. We now review how GES manipulates MECs in
practice and then show how to make it more efficient.

\subsection{Manipulating MECs with CPDAGs}\label{sec:manipulating_mecs}
MECs are sets of DAGs whose size can grow exponentially with the number of nodes $d$
\citep{he2015counting}. To manipulate them practically, GES builds on the
following theorem.

\begin{theorem}[\citep{verma1991equivalence}]
    Two DAGs are Markov equivalent if and only if they have the same skeletons and the
    same v-structures. 
    \label{thm:mec}
 \end{theorem} 

 The \textit{skeleton} of a graph is the undirected graph obtained
 by removing the direction of all edges; a \textit{v-structure} is a triple of nodes
such that $x \rightarrow y \leftarrow z$ with no edge between $x$ and
 $z$.\looseness=-1

\Cref{thm:mec} shows that all the graphs of a MEC share the same skeleton and differ
only on edges that can be reversed without changing the set of v-structures. So within a
MEC, some edges are consistently oriented in one direction, while others may have
different orientations between graphs. They are respectively called \textit{compelled}
and \textit{reversible} edges.

\parhead{CPDAGs.}
Each MEC can be represented by a \textit{partially directed acyclic
graph} (PDAG). A PDAG is a graph with both directed and undirected edges and no cycles of
directed edges. The \textit{canonical PDAG of a MEC} contains all the compelled edges as directed
edges and all the reversible edges as undirected edges (see \Cref{fig:mec}). A PDAG that is the canonical
representation of a MEC is called a \textit{completed PDAG} (CPDAG).
A PDAG $P$ that is not a CPDAG but has the same skeleton and v-structures as another CPDAG $P'$
can be transformed into $P'$ with a method called 
\textit{completing} the PDAG \citep{meek1995causal,chickering2002learning}.

We will use the following terminology when discussing a PDAG. For node $x$: its \textit{neighbors}
$\Ne(x)$ are its neighbors from undirected edges, its \textit{children} $\Ch(x)$ are its
children from directed edges, its \textit{parents} $\Pa(x)$ are its parents from
directed edges, and its \textit{adjacent} nodes $\Ad(x)$ are any of all three. 
A \textit{semi-directed path} from $x$ to $y$ is a path from $x$ to $y$ with edges that
are either undirected or directed toward the direction of $y$. A \textit{clique} is a
set of all adjacent nodes.

\parhead{Operators on CPDAGs.}
GES associates each operation on a MEC $M$  with an operator acting on its CPDAG $P$,
such that an operation changing $M$ into $M'$ is associated with an operator changing
$P$ into $P'$, the CPDAG of $M'$. 

For insertions, the operators used by GES are of the form \text{Insert}$(x,y,T)$ where
$x,y \in V$ and $T \subset V$. The action of Insert$(x,y,T)$ on $P$ is to insert the
edge $x \rightarrow y$, orient any undirected edges $t - y$ as $t \rightarrow y$ for $t
\in T$ and finally complete the resulting PDAG into a CPDAG. 

Given a MEC $M$ and its CPDAG $P$, \citet{chickering2002optimal} shows that there is a
bijection between (a) the set of possible insertions on $M$ and (b) the set of operators
Insert$(x,y,T)$ satisfying the following validity conditions relative to $P$:
\begin{align}
    \bm{I1.}\quad&x \not\in \mathrm{Ad}(y). \label{eq:valid1} \\
    \bm{I2.}\quad&T \subset \mathrm{Ne}(y) \setminus \mathrm{Ad}(x). \label{eq:valid2} \\
    \bm{I3.}\quad&(\mathrm{Ne}(y) \cap \mathrm{Ad}(x)) \cup T \text{ is a clique.} \label{eq:valid3} \\
    \bm{I4.}\quad&\text{All semi-directed paths from } y \text{ to } x \text{ have a node in }\nonumber \\
    &(\mathrm{Ne}(y) \cap \mathrm{Ad}(x)) \cup T. \label{eq:valid4} 
\end{align}
Insert operators satisfying these conditions, with $\Ad, \Ne,$ $\Pa,$ clique and paths
computed in $P$, are called \textit{valid} for $P$. To navigate from one MEC to another
with an insertion, GES applies the corresponding valid Insert operator from one CPDAG to
another.

\parhead{Score of Operators.}
The increase in score after an Insert operation can be efficiently computed when the
score is BIC. Indeed, the BIC for a graph $G$ equivalently rewrites as: 
\begin{equation}
    \textstyle
    S(G; \data) = \sum_{j=1}^d s(j, \Pa_j^G ; \data), \label{eq:bic-decomp}
\end{equation}
where $s(j, \Pa_j^G ; \data)$ is called the \textit{local score} of $j$ and equals:
\begin{equation}
    \sum\limits_{i=1}^n \log p_{\hat
\theta}(x^i_j | x^i_{\Pa_j^G}) - \frac{\alpha}{2}\log n \cdot |\Pa_j^G|.
\end{equation}
A score decomposing as \Cref{eq:bic-decomp} is called \textit{decomposable}.

With a decomposable score, the increase in score for an operator Insert($x,y,T$) applied
to $P$ is:
\begin{multline}
    \delta = s(y, (\Ne(y) \cap \Ad(x)) \cup T \cup \Pa(y) \cup \{x\}) \\
    - s(y, (\Ne(y) \cap \Ad(x)) \cup \Pa(y )),
    \label{eq:insert-score}
\end{multline}
where each term $\Ad, \Ne, \Pa$ is computed relative to $P$. For convenience, we say
that $\delta$ is the \textit{score} of the operator.

Similar derivations are made for Delete and Reversal in \citet{chickering2002optimal}
 and \citet{hauser2012characterization} (reversal is called turning). We review them in
 \Cref{appendix:sec:ges_parametrization}.

\parhead{GES with CPDAGs.}
In sum, GES implements \Cref{alg:ges-vanilla} using CPDAGs and operators. It begins with
the empty CPDAG, identifies all the Insert (or Delete, Reversal) that are valid for the
current CPDAG, computes their scores, applies the best one if it has a positive score,
and repeats.

XGES could proceed similarly. However, whether for GES or XGES, constructing the list of
valid operators and scoring them at each step is computationally expensive. We now turn
to new ways to more efficiently implement these operations.

\subsection{Efficient Algorithmic Formulation}
\label{sec:efficient_algorithmic_formulation}
When applying an operator on $P$ to form $P'$, the validity conditions of the other
operators (\Cref{eq:valid1,eq:valid2,eq:valid3,eq:valid4}) can become valid or invalid.
Similarly, the score of the other operators in \Cref{eq:insert-score} can
change. Yet, as noticed in \citet{ramsey2017million}, only a few edges changed from $P$ to $P'$. As a result, most
other operators that were computed for $P$ but not applied remain valid operators for
$P'$. Similarly, the scores of most operators remain identical.

Each step of XGES involves the following sub-steps:
\begin{enumerate}
    \item Start with a CPDAG $P$ and a list of candidate operators $\mathcal{C}$,
    where $\mathcal{C}$ is guaranteed to include all the valid operators for $P$, and
    their scores.
    \item Choose the best operator $O^*$ from
    $\mathcal{C}$ using XGES's heuristic (deletion before reversal, before insertion).
    \item Verify that $O^*$ is valid for $P$, otherwise re-run the heuristic on
    $\mathcal{C} \setminus \{O^*\}$ until a valid operator is found.
    \item Apply $O^*$ to $P$ to form $P'$ and add to $\mathcal{C}$ all the operators that
    became valid for $P'$, with their scores. Return to 1, with $P \leftarrow P'$, as we
    have just guaranteed that $\mathcal{C}$ includes all the valid operators for $P$.
\end{enumerate}

The operators that became invalid for $P'$ are not removed from $\mathcal{C}$. It is more efficient to leave them in the list and only check the validity of an operator in step 3 just before applying it (and discarding it if invalid). Indeed, if we recheck the validity of all operators after each operation, a single operator will be rechecked at each step until it is applied, instead of being checked only once before being applied.

No steps were included to recompute the scores of any operators in
$\mathcal{C}$. We explain how we can avoid it next.

\subsubsection{Updating the Score of Operators.}
\label{sec:updating_score}
To avoid recomputing the scores of operators at each step, we
change the parametrization of the
operators to make their scores independent of the CPDAG they are applied to. 

We parametrize each Insert by an additional set $E \subset V$ and an extra validity
condition that completes \Cref{eq:valid1,eq:valid2,eq:valid3,eq:valid4}: 
\begin{flalign}
    ~~~\bm{I5.}\quad & E = (\Ne(y) \cap \Ad(x)) \cup T \cup \Pa(y).&&
\end{flalign} 
The score of Insert($x,y,T,E$) from \Cref{eq:insert-score} becomes $s(y, E \cup \{x\}) -
s(y, E )$, which only depends on the Insert parameters. We reparametrize Delete and
Reversal operators similarly in \Cref{appendix:sec:xges_parametrization}.

With \citet{chickering2002optimal}'s parametrization, the score of an Insert would
change if $(\Ne(y) \cap \Ad(x)) \cup T \cup \Pa(y)$ changes. Now, with $E$ as a fixed
parameter of the operator, it is the status of condition \textbf{I5} that would change.
We turned a change in score into a change in validity.

We now turn to efficiently update the valid operators.

\begin{table}[t]
    \centering
    \begin{tabular}{ L{1.7cm}C{3.4cm}C{1.8cm} }
    \toprule
    Pre-update & $a \quad b$ & $a
    \rightarrow b$ \\
    Post-update& $a - b$ &  $a- b$  \\
    \midrule
    Necessary conditions
    &\begin{itemize}[leftmargin=*, itemsep=-1pt]
        \vspace*{-4mm}
        \item[] $y \in \{a,b\}$ 
        \item[or] $y \in \Ne(a)\cap \Ne(b)$
        \item[or] $(x = a) \wedge (y \in \Ne(b))$
        \item[or] $(x = b) \wedge (y \in \Ne(a))$
    \end{itemize}
    & \begin{itemize}[leftmargin=*]
        \item[] $y \in \{a,b\}$
    \end{itemize}
    \\[-3mm]
    \bottomrule
    \end{tabular}
    \caption{Necessary conditions for an Insert($x,y,T,E$) to become valid after the
    $(a,b)$ update. Excerpt of \Cref{tab:operator_updates} from
    \Cref{appendix:sec:efficient_algorithmic_formulation} with only two types of
    updates. }
    \label{tab:operator_updates_example}
\end{table}

\subsubsection{Updating the Validity of Operators.}
After updating a CPDAG $P$ into $P'$, our goal is to efficiently add to $\mathcal{C}$
the operators that became valid for $P'$.

To do so, we decompose the update from $P$ to $P'$ into a succession of single edge
updates $P_1, \hdots P_k$, with $P_1=P$, $P_k = P'$ and where $P_i$ and $P_{i+1}$ only
differ on the orientation or presence of a single edge, e.g. $a \rightarrow b$ vs $a -
b$. We then have the following theorem.

\begin{theorem}
    Write $P_1, \hdots P_k $ a sequence of single edge updates that transforms $P$ into
    $P'$. Take an operator $O$ that is invalid for $P$ and becomes valid for $P'$ and
    write $\{c_1, \ldots, c_m \}$ its validity conditions, e.g. $\bm{I1}$ to $\bm{I5}$
    for an Insert. Then there exists $i^* \in \{1,k-1\}$ and one validity condition
    $c_{j^*}$ such that $c_{j^*}$ is false for $P_{i^*}$, true for $P_{i^*+1}$, and all
    other conditions $c_{j} \neq c_{j^*}$ are true for $P_{i^*+1}$.
    \label{thm:update_validity}
\end{theorem}
\begin{proof}
    All $c_j$ are true for $P'$ i.e. $P_k$. So let us step back from $P'$ to $P$ until
    one of the conditions $c_{j^*}$ becomes false for some $P_{i^*}$. Such an $i^{*}$
    must exist since some condition is false for $P$ i.e. $P_1$. $P_{i^*}$ and $c_{j^*}$
    satisfy the theorem.
\end{proof}

With \Cref{thm:update_validity}, we can efficiently update $\mathcal{C}$ if we can
identify which operators are susceptible to having one of their conditions
become true after single-edge updates.

In \Cref{appendix:sec:efficient_algorithmic_formulation} we study the necessary
conditions on the parameters of an operator to have one of its validity conditions
become true after a single-edge update. We report the necessary conditions for all
validity conditions of all operators against all types of edge updates in
\Cref{tab:operator_updates} in \Cref{appendix:sec:efficient_algorithmic_formulation}. We
provide an excerpt in \Cref{tab:operator_updates_example} with only two types of edge
updates, for the Insert operator only, and where we grouped the necessary conditions for
each validity condition into a single set of necessary conditions (with or).

For example, if edge $a \rightarrow b$ is changed into $a - b$,
\Cref{tab:operator_updates_example} shows that the only Insert($x,y,T,E$) that 
can become
valid are those with $y \in \{a,b\}$. 
If the edge $a - b$ is changed into $a \rightarrow
b$, then the necessary conditions for an Insert operator to become valid are more
involved but still efficient.\looseness=-1

In sum, we can efficiently update $\mathcal{C}$ after each CPDAG update using
\Cref{tab:operator_updates} in \Cref{appendix:sec:efficient_algorithmic_formulation}.

\subsubsection{XGES Implementation.}

We implement the efficient algorithmic formulation of XGES-0 and XGES  at \href{https://github.com/ANazaret/XGES}{https://github.com/ANazaret/XGES}
 We provide
code in C++ and Python.

% !TEX root = ../main.tex

\section{Empirical Studies}\label{sec:experiments} We compare the XGES heuristics to
different variants of GES. We find that
XGES recovers causal graphs with significantly better accuracy and up to 10 times faster.

\subsection{Evaluation Setup}
\label{sec:experiments:evaluation_setup}
\parhead{Data Simulation.}
We simulate CGMs and data for different numbers of variables $d$, edge density $\rho$
(average number of parents) and number of samples $n$. 
We first draw a random DAG $G^*$
from an Erdos-Renyi distribution. We then obtain $p^*$ by choosing each conditional
distribution $p^*(x_i \mid x_{\Pa^i_{G^*}})$ as a Gaussian $x_i \sim
\mathcal{N}(W_i^\top x_{\Pa^i_{G^*}}, \varepsilon_i)$ where $W_i, \varepsilon_i$ are
random selected. To ensure faithfulness, we sample $W_i$ away from $0$. More details are
in \Cref{appendix:sec:simulated_data}.

\parhead{Baseline Algorithms.}
We compare our algorithms against GES without reversals (GES), and with reversals
(GES-r, a.k.a GIES), using the C++ implementation in the R package \texttt{pcalg}
\citep{kalisch2012causal}. We also include fast-GES (fGES) from the Java software Tetrad
\citep{ramsey2017million}.
An additional baseline, OPS, is provided in \Cref{appendix:sec:empirical:ops}.

\parhead{Evaluation Metrics}
We evaluate the algorithms with the structural Hamming distance on MECs (SHD) between
the method's results $\hat M$ and the ground-truth MEC $M^*$ \citep{peters2014causal}. The SHD is the number
of different edges between the CPDAGs of $\hat M$ and $M^*$. 
We also consider causal discovery as a binary classification task, where $\hat M$ predicts the presence of edges in $M^*$. We report the F1 score, precision, and recall for this task. Error bars are computed over multiple random datasets (seeds) and reported as bootstrapped 95\% confidence intervals \citep{waskom2021seaborn}.

\begin{figure}[t]
    \centering
    \includegraphics[width=\linewidth]{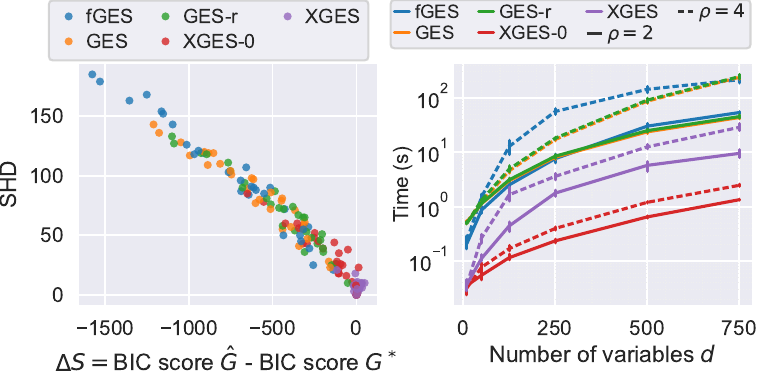}
    \caption{
    (left) The BIC scores of the graphs returned by each method are strongly correlated with the SHD to ground truth (shown for $d=50$, $\rho=3$, 30 seeds). XGES finds the highest scores and lowest SHDs. (right) Runtime of GES and XGES for a wide
    range of $d$. XGES-0 is up to 30 times faster than GES, and XGES up to 10 times
    faster. fGES may have overhead due to Java while other methods are in C++.}
    \label{fig:simulation_correlation_speed}
\end{figure}

\subsection{Results}
\parhead{General Performance.} % when we vary d, rho and fix n=10000, alpha=2
In \Cref{fig:simulation_main}, we find that XGES-0 and XGES outperform all baselines.
The improvement is more significant for larger density $\rho$ and larger $d$. The
conclusions are identical with precision and recall in \Cref{fig:simulation_main_precision_recall} of \Cref{appendix:sec:empirical:f1_score},
which are both improved by XGES. We emphasize that even though XGES favors deleting
edges, the proportion of true edges recovered by XGES is higher than GES (the recall).
We also report the F1 metric in \Cref{appendix:sec:empirical:f1_score}. We note that the
performance of fGES is slightly worse than GES. We explain in
\Cref{appendix:subsec:fast_ges} that one of the optimizations of fGES removes
some valid insertions.

\parhead{Choice of Metrics.}
\Cref{fig:simulation_correlation_speed} (left) shows that the BIC scores of the graphs returned by each method are strongly correlated with the SHD to ground truth. This is a comforting observation: maximizing the BIC score on finite data is indeed a good proxy for minimizing SHD to ground-truth.

\parhead{Impact of the Edge Density $\rho$.}
In \Cref{fig:simulation_d_and_rho}, we see that when $\rho=1$ (a very sparse graph where
nodes have $\rho=1$ parent on average), then all methods perform similarly well. The
advantage of XGES over GES is visible as soon as $\rho=2$ and widens as $\rho$ increases, see also
\Cref{fig:simulation_main}.

\parhead{Robustness to the Sample Size $n$.}
In \Cref{fig:simulation_n_and_alpha}~(left), we vary the number of samples $n$ and fix
$d=50$ and $\rho=3$. XGES's performance improves with $n$, coherent with
\Cref{th:local_consistency}. In contrast, GES and its variants are hurt when $n$ increases
beyond $10^4$. But this is not incoherent with GES's correctness in the limit of
infinite data. Instead, this reveals that the finite sample behavior of GES is nontrivial and that GES may require very large $n$ -- beyond what is practical -- to perform well.\looseness=-1

In \Cref{fig:simulation_double_descent}, we study sample sizes up to $n=10^8$ on a small graph with $d=15$ and $\rho=2$. We find again that GES worsens around $10^4$ samples, but this time, it improves again after $10^5$ samples, thereby exhibiting a double descent behavior. We discuss it in more detail in \Cref{appendix:sec:empirical:double_descent}.

% figure: simulation_n and simulation_alpha
\begin{figure}
    \centering
    \includegraphics[width=\linewidth]{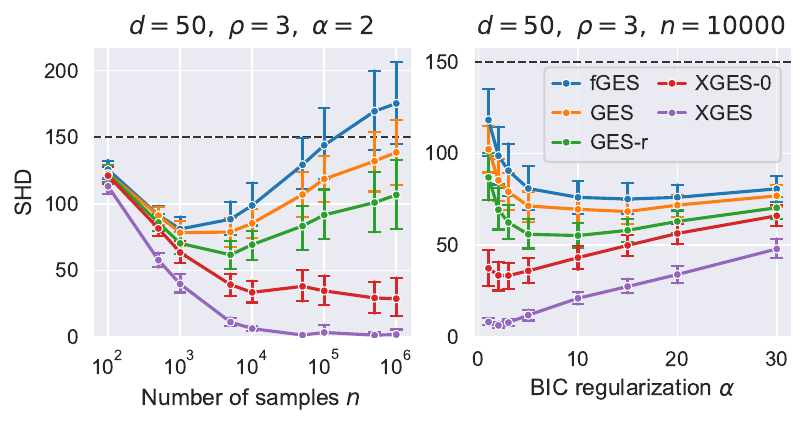}
    \caption{Performance of GES and XGES when varying (left) the number of samples $n$, and
    (right) the regularization strength $\alpha$. Increasing $n$ improves XGES while it hurts GES and its
    variants. Increasing
    $\alpha$ initially improves GES but eventually hurts all methods. The dashed lines indicate the number of edges of the true graph. Error bars over 30 seeds.\looseness=-1}
    \label{fig:simulation_n_and_alpha}
\end{figure}

\parhead{Robustness to the Regularization Strength $\alpha$.}
In \Cref{fig:simulation_n_and_alpha} (right), we vary
$\alpha$ and fix $d=50$, $\rho=3$ and $n=10000$. We find that increasing $\alpha$ helps
GES from $\alpha=1$ to $\alpha=10$ but then hurts it. No value of $\alpha$ enables GES to
catch up to XGES. Echoing \Cref{sec:ges_failure}, we conclude that the solution to GES's
over-inserting is not to change the score function, but to change the search strategy,
as XGES does.

\parhead{Robustness to the Data Simulation.}
We vary the procedure to sample the weights $W_i$ in two ways: changing the scale
and changing the shape of their distribution. We report the results in
\Cref{appendix:sec:empirical:impact_of_simulation} with \Cref{fig:simulation_epsilon,fig:simulation_main_negative}.
We find similar conclusions as in \Cref{fig:simulation_main}.

\parhead{Implementation Speed.}
We measure the runtime of the different methods for a wide range of $d$ and 
$\rho \in \{2,4\}$ in \Cref{fig:simulation_correlation_speed} (right). We find that XGES-0 is up to 30 times
faster than GES, and XGES up to 10 times faster. While fGES's slower runtime may be 
attributed to Java overhead, the other methods are implemented in C++.
Higher densities slow down all methods, with a stronger impact on GES, which is coherent
with GES over-inserting in denser graphs. 

We find the same conclusions by reporting the number of calls to the scoring function as another measure
of efficiency in \Cref{appendix:sec:empirical:number_of_calls}. Interestingly, even though XGES repeatedly applies XGES-0, it only makes around one order of magnitude more BIC score evaluations than XGES-0.

\section*{Conclusion and Future Work}
We introduced XGES, an algorithm that significantly improves on GES. With XGES, we can learn larger and denser graphs from data. XGES offers several avenues for future work. One direction is to study its finite sample guarantees. A second is to study its applicability to more complex model classes beyond linear models. Another is to relax its assumptions: e.g. unfaithful graphs, or scores that are not asymptotically locally consistent \citep{schultheiss2023pitfalls}. Finally, its efficient implementation could be used to analyze large real-world datasets.\looseness=-1

\begin{acknowledgements}
We thank the anonymous reviewers for their helpful comments.
A.N. was supported by funding from the Eric and Wendy Schmidt Center at the Broad Institute of MIT and Harvard, and the Africk Family Fund. 
D.B. was funded by NSF IIS-2127869, NSF DMS-2311108, NSF/DoD PHY-2229929, ONR N00014-17-1-2131, ONR N00014-15-1-2209, the Simons Foundation, and Open Philanthropy.
\end{acknowledgements}

% References
\bibliography{main}

\newpage

\onecolumn

\title{Extremely Greedy Equivalence Search\\(Supplementary Material)}
\maketitle
\appendix
% !TEX root = ../main.tex

\section{Empirical Studies}

\subsection{Sample Size and Regularization Strength}
\label{appendix:sec:empirical:sample_size_and_regularization_strength}
We reproduce figure \Cref{fig:simulation_n_and_alpha} from the main text in
\Cref{fig:simulation_n_and_alpha_copy}, which shows the impact of the sample size $n$
and the regularization strength $\alpha$ on the performance of GES and XGES variants. We
complete it with \Cref{fig:simulation_n_and_alpha_edges} to show the corresponding
number of predicted edges. It indeed reveals that increasing $n$ causes GES and its
variants to over-insert.

\begin{figure}[h]
    \centering
    \begin{subfigure}{0.45\linewidth}
        \centering
        \includegraphics[width=\linewidth]{fig/simulation_n_and_alpha.pdf}
        \caption{Evolution of the performance of GES and XGES variants with the number of
        samples $n$.}
        \label{fig:simulation_n_and_alpha_copy}
    \end{subfigure}
    \hfill
    \begin{subfigure}{0.45\linewidth}
        \centering
        \includegraphics[width=\linewidth]{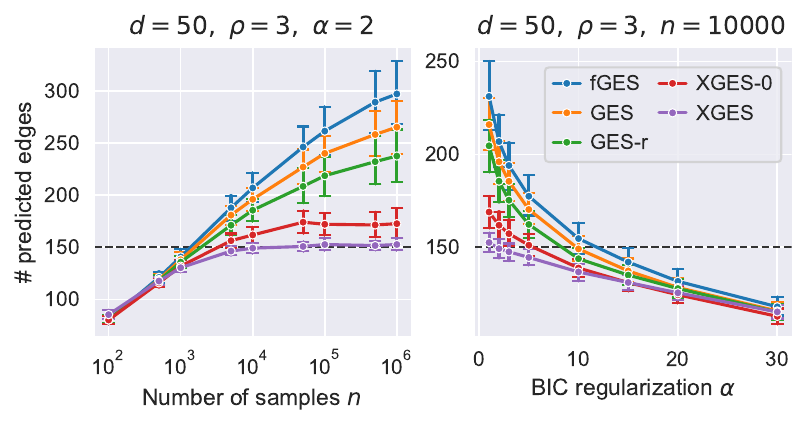}
        \caption{Evolution of the number of predicted edges with the number of samples $n$.}
        \label{fig:simulation_n_and_alpha_edges}
    \end{subfigure}
    \caption{Evolution of the performance of GES and XGES variants with (left) the
    number of samples $n$, and (right) the regularization strength $\alpha$. Increasing
    $n$ hurts GES and its variants (as it reduces the BIC regularization and
    over-insertion is exacerbated) while it helps XGES. Increasing the BIC
    regularization with $\alpha$ helps GES but without letting it catch up with XGES.
    Missing points indicate the method did not run.}
\end{figure}

\begin{figure}[h]
    \centering
    \begin{subfigure}{0.49\linewidth}
        \centering
        \includegraphics[width=\linewidth]{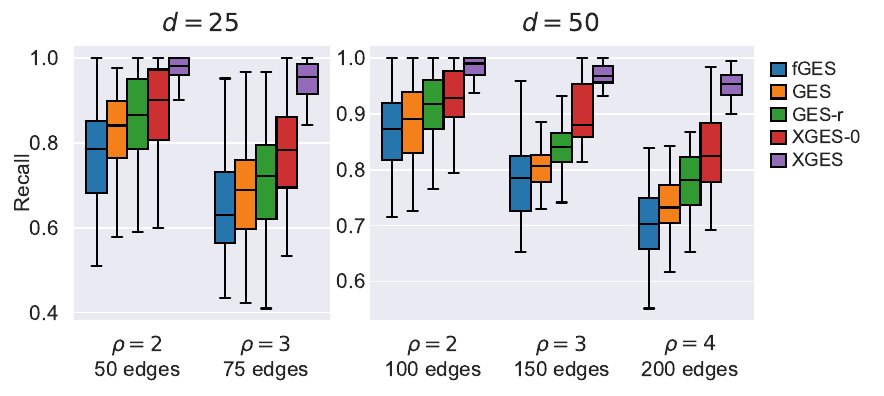}
        \caption{Recall of GES and XGES on simulated data. }
    \end{subfigure}
    \hfill
    \begin{subfigure}{0.49\linewidth}
        \centering
        \includegraphics[width=\linewidth]{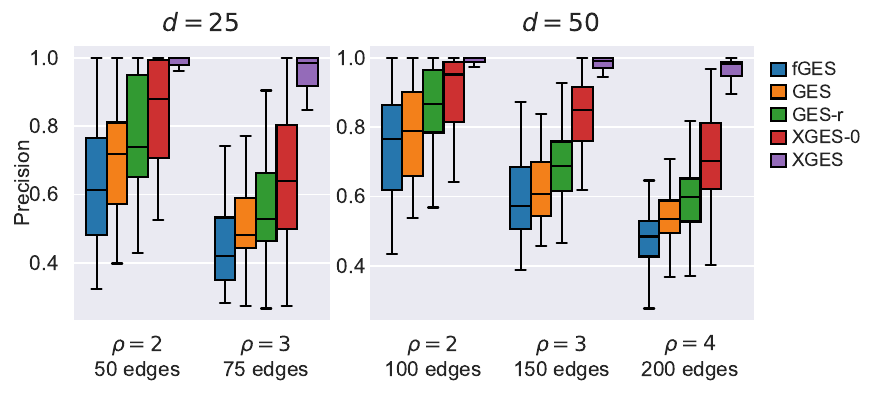}
        \caption{Precision of GES and XGES on simulated data.}
    \end{subfigure}
    \begin{subfigure}{0.49\linewidth}
        \centering
        \includegraphics[width=\linewidth]{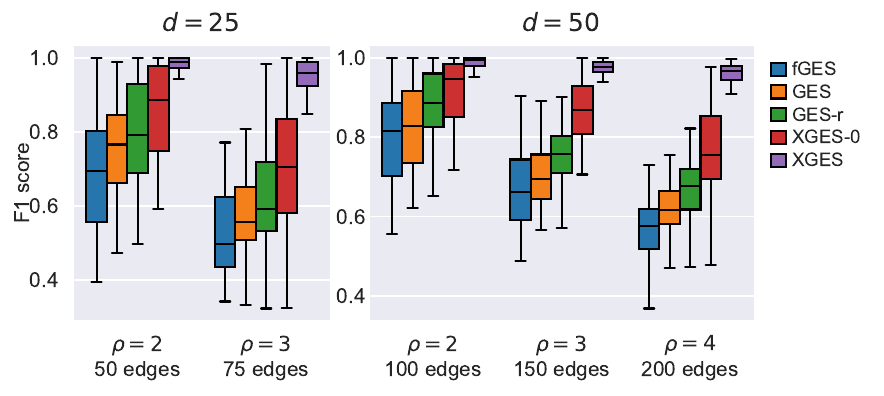}
        \caption{F1 of GES and XGES on simulated data.}
    \end{subfigure}
    \caption{Performance of GES and XGES on simulated data measured with precision, recall and F1. With $n=10000, \alpha=2$ and 30 seeds (same data as \cref{fig:simulation_main}). XGES outperforms other methods in the three metrics.}
    \label{fig:simulation_main_precision_recall}
\end{figure}

\begin{figure}
    \centering
    \includegraphics[width=0.7\linewidth]{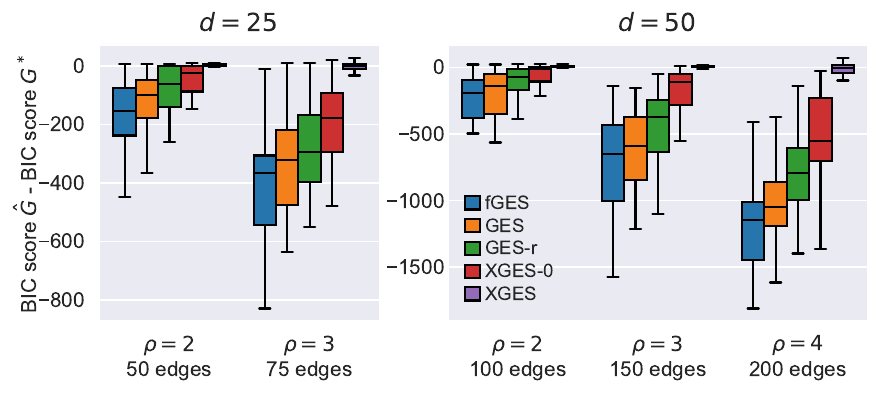}
    \caption{Difference in score between the true graph/MEC $G^*$ and the graph/MEC returned by GES and XGES on simulated data. XGES returns graphs with higher scores than GES.}
    \label{fig:simulation_main_score_diff}
\end{figure}

\subsection{Complementary metrics: precision, recall, F1 score, BIC score.}
\label{appendix:sec:empirical:f1_score}
We complement \Cref{fig:simulation_main} from the main text with
\Cref{fig:simulation_main_precision_recall} to show the F1 score of GES and XGES for the same simulated data, as well as the precision/recall breakdown. We also report the BIC score of the graphs returned by GES and XGES in \Cref{fig:simulation_main_score_diff}. 
We find similar conclusions as in the main text.

The methods return a MEC $\hat M$ to predict the true MEC $M^*$. The F1 score, precision
and recall are defined for binary classification problems. For each ordered pair of
nodes $(i,j)$ we say $M$ contains $(i,j)$ if $(i,j)$ is directed in $M$ from $i$ to $j$
or if $(i,j)$ is undirected in $M$ $(i,j)$ (but not if $(j,i)$ is directed in $M$ from $j$ to $i$). Then, the binary classification problem is to predict for each $(i,j)$ if
$M^*$ contains $(i,j)$ or not, predicted by whether $\hat M$ contains $(i,j)$ or not.

\clearpage
\newpage
\subsection{More variables and edge densities}
\label{appendix:sec:empirical:more_variables_and_edge_densities}
We show the performances of the methods on more combinations of $d$ and $\rho$ in
\Cref{fig:simulation_d_and_rho}. We fixed $n=10000, \alpha=2$.
We find similar trends as in \Cref{fig:simulation_main} and \Cref{fig:simulation_correlation_speed}.

For edge density $\rho=1$, we find no significant difference between GES and XGES. This
is expected as the true graph is really sparse and the methods are less likely to
encounter local optima.
\begin{figure}[h]
    \centering
    \includegraphics[width=0.98\linewidth]{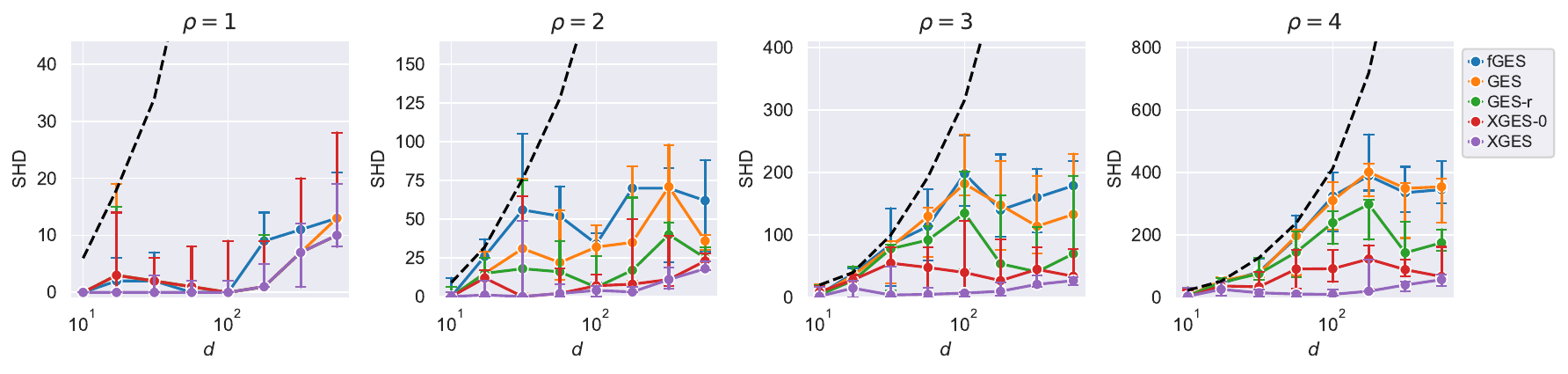}
    \caption{Performance of GES and XGES on more combinations of $d$ and $\rho$. We fixed
    $n=10000, \alpha=2$ and we report error bars and averages over 5 seeds. XGES consistently outperforms GES and its variants. The dashed lines indicate the number of true edges.}
    \label{fig:simulation_d_and_rho}
\end{figure}

\subsection{Number of calls to the scoring function}
\label{appendix:sec:empirical:number_of_calls}
We show the number of calls to the scoring function for the different methods in
\Cref{fig:simulation_number_of_calls}. We find that even though XGES repeatedly applies
XGES-0, it only makes around one order of magnitude more BIC score evaluations. fGES was
not included because we could not extract the number of calls from the Tetrad
implementation. 

\begin{figure}[h]
    \centering
    \includegraphics[width=0.4\linewidth]{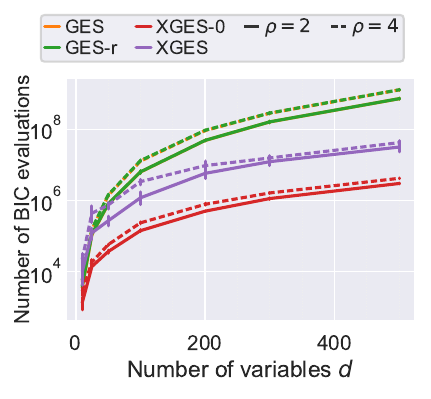}
    \caption{Number of calls to the scoring function for the different methods. We fixed
    $n=10000, \alpha=2$, and we report error bars and averages over 5 seeds.}
    \label{fig:simulation_number_of_calls}
\end{figure}

\subsection{The OPS variant}
\label{appendix:sec:empirical:ops}
\citet{chickering2002optimal} evaluated GES against a variant called OPS that also
    considered insertions and deletions simultaneously. But OPS did not prioritize
    deletions over insertions, resulting in limited changes to GES in their experiments.
    We show the performance of OPS in \Cref{fig:simulation_ops}. We
    corroborate the results of \citet{chickering2002optimal} that OPS has performances
   similar to GES.

 Mathematically, deleting an edge can only increase the BIC score by at most $\frac{\alpha}{2} \log n$ which is usually smaller than the increase from inserting an edge.
 Hence even if OPS considers deletions and insertions ``together'', we conjecture that most deletions are only considered at the end of the search because insertions have higher scores and are applied first. OPS then encounters the same local maxima as GES. 
 This highlights the importance of prioritizing deletions over insertions with XGES.\looseness=-1

 \begin{figure}
    \centering
    \includegraphics[width=0.6\linewidth]{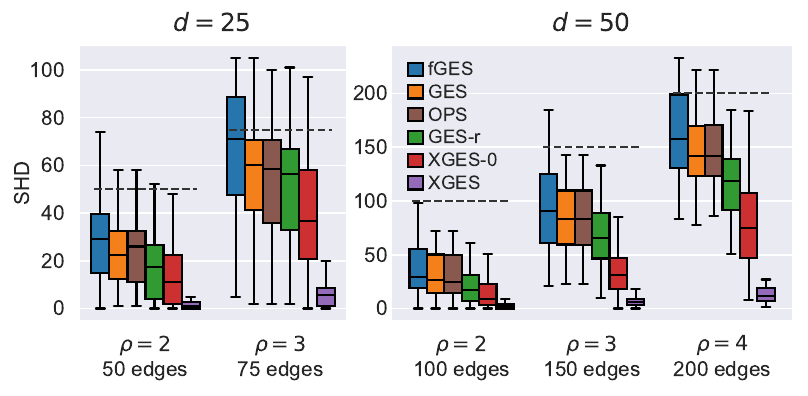}
    \caption{Same experiment as \Cref{fig:simulation_main} but with the addition of the OPS variant. OPS performs very similarly to GES, suggesting that XGES-0's heuristic favoring deletions over insertions in XGES-0 is important. }
    \label{fig:simulation_ops}
 \end{figure}

\subsection{The double descent of GES with the sample size}
\label{appendix:sec:empirical:double_descent}
We show the performance of GES and XGES on a small graph with $d=15$ and $\rho=2$ for
sample sizes up to $n=10^8$ in \Cref{fig:simulation_double_descent}. XGES monotonically and quickly improves to an SHD of 0. We find that GES
improves from $n=10^2$ to $n=10^4$ but worsens around $n=10^4$. It eventually improves
again after $10^6$. The error bands are bootstrapped 95\% confidence intervals over 30 
seeds. We chose the graph to be small so that we could observe the second descent of
GES with sample sizes up to $10^8$. However, we doubt that such large sample sizes
are practical. We further believe that the issue worsens with larger graphs. (\textit{Note: We reimplemented GES and GES-r to scale to $n=10^8$, we verified that we obtained the same results as the original implementations for $n$ up to $10^6$. See parameter \texttt{-b} in the XGES code.})

When the sample size increases, the strength of the BIC regularization relative to the likelihood decreases in $\frac{\log(n)}{n}$. We hypothesize that GES worsens because
that decrease might lead to over-inserting, and the encounter of local optima.

\begin{figure}[h]
    \centering
    \includegraphics[width=0.7\linewidth]{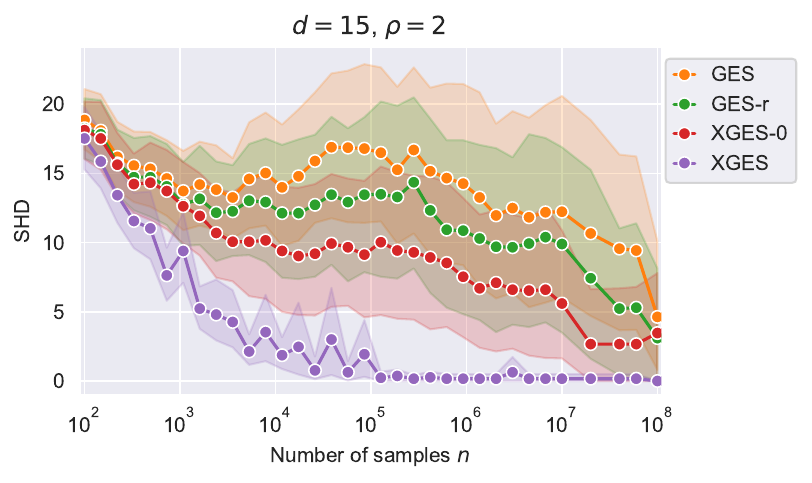}
    \caption{Performance of GES and XGES on a small graph with $d=15$ and $\rho=2$ for
    sample sizes up to $n=10^8$. We find that GES worsens around $n=10^4$ but improves
    after $n=10^6$. It exhibits a double descent behavior with the sample size.}
    \label{fig:simulation_double_descent}
\end{figure}

\subsection{The impact of the simulation}
\label{appendix:sec:empirical:impact_of_simulation}

We describe the simulation procedure and show the impact of changing it. 

\subsubsection{The simulation procedure}
\label{appendix:sec:simulated_data}

In the experiments, we use the following procedure to obtain simulated data.

\begin{enumerate}
    \item Select $d$ and $\rho$.
    \item Draw a DAG $G^*$ as follows:
        \begin{itemize}[leftmargin=*]
            \item $G^* \sim \text{Erdos-Renyi}(\#nodes=d, p=\frac{2\rho}{d-1})$.
            \item Orient each edge $(i,j)$ from $\min(i,j)$ to $\max(i,j)$, that forms a
            DAG.
            \item Draw a permutation $\sigma$ of $[d]$ and relabel each node $i \mapsto
            \sigma(i)$. 
        \end{itemize}
    \item Choose a distribution $p^*$ as follows for each $i
    \in [d]$:
    \begin{itemize}[leftmargin=*]
        \item Sample weights $W_i$ away from $0$ to ensure faithfulness.
        \begin{enumerate}
            \item In most simulations, we draw $W_i \in [1,3]^{\left|\Pa^i_{G^*}\right|}$ from
            $\mathcal{U}([1,3])$.
            \item In \Cref{fig:simulation_main_negative}, we change the simulation to sample weights $W_i \in
            [-3,-1] \cup [1,3]$, by additionaly drawing a Bernoulli variable $b_i \sim
            \mathcal{B}(0.5)$ and setting $W_i \leftarrow -W_i$ if $b_i=1$.
        \end{enumerate}

        \item $L_1$-normalize $W_i$ as $W_i \leftarrow W_i / \sum_j |W_{ij}|$.
        \item Draw the scale of the Gaussian $\varepsilon_i \sim \mathcal{U}(0, \varepsilon_{\text{max}})$. 
        \begin{enumerate}
            \item By default in most plot $\varepsilon_{\text{max}}=0.5$. 
            \item We vary it in \Cref{fig:simulation_epsilon}.
        \end{enumerate}

        \item Define $p^*(x_i \mid x_{\Pa^i_{G^*}}) = \mathcal{N}(W_i^\top
        x_{\Pa^i_{G^*}}, \varepsilon_i^2)$.
    \end{itemize}
    \item Draw $n$ i.i.d. samples from $p^*$.
\end{enumerate}

\subsubsection{Varying the scale of the Gaussian noise}

We show the impact of varying the scale of the Gaussian noise $\varepsilon_{\text{max}}$
in \Cref{fig:simulation_epsilon}. As expected, GES and XGES are almost not impacted by
the scale of the noise. This is expected, as the methods explicitly model the noise variable.

\begin{figure}[h]
    \centering
    \includegraphics[width=0.9\linewidth]{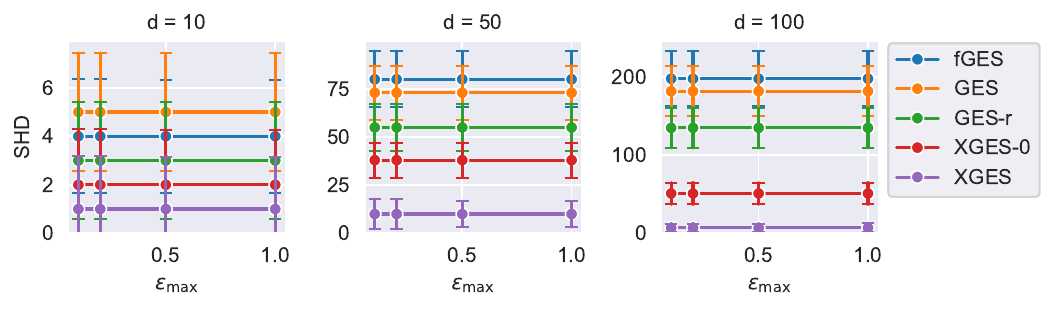}
    \caption{Performance of GES and XGES on different scales of the Gaussian noise. We
    fixed $n=10000, \alpha=2$.}
    \label{fig:simulation_epsilon}
\end{figure}

\clearpage
\subsubsection{Different sampling of weights}
We find similar conclusions when we change the sampling of weights from $W_i \in [1,3]$
to $W_i \in [-3,-1] \cup [1,3]$ . We show the results in
\Cref{fig:simulation_main_negative}.

\begin{figure}[h]
    \centering
    \includegraphics[width=0.5\linewidth]{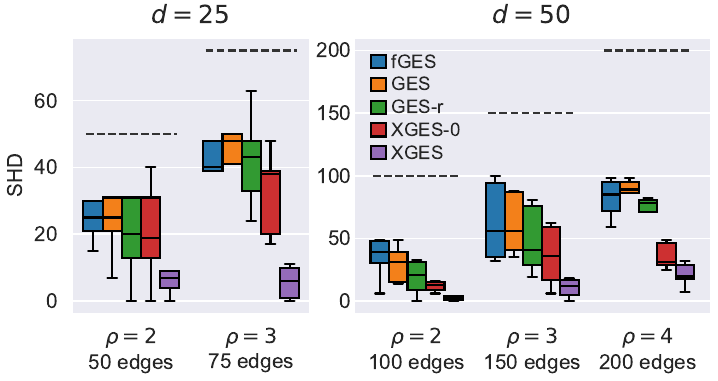}
    \caption{Performance of GES and XGES similar to \Cref{fig:simulation_main} but with
    a different sampling of weights $W$. $n=10000, \alpha=2$. }
    \label{fig:simulation_main_negative}
\end{figure}

% The Tetrad implementation, called fGES, considers a single Insert(X,Y,T S) per (X,Y) as
% the best insert over all T (in term of score, not validity). Yet, the best T might lead
% to an invalid insert, and so all X->Y insert (potentially with another T), will be
% ignored. 

\clearpage
\section{Theoretical Details}
We detail the theoretical guarantees of GES and XGES. We first recall the problem only
in terms of MECs and not CPDAGs. We review the theoretical guarantees of GES and prove
those of XGES. We then show how the MEC formulation can be translated into a CPDAG
formulation for Deletes and Reversals.

\parhead{Notations and Vocabulary.}
\begin{itemize}[noitemsep]
    \item $d$: number of variables.
    \item $G$: a DAG over $d$ variables.
    \item $P$: a PDAG over $d$ variables.
    \item $M$: a MEC over $d$ variables.
    \item $\mathcal{M}$: space of all MEC over $d$ variables.
    \item $S$: a locally consistent score that is score equivalent. 
    \item $S(G)$ or $S(M)$ or $S(P)$: the score of a DAG $G$, a MEC $M$, or a CPDAG $P$.
    We hide the dependence in $\data$ for simplicity.\looseness=-1
    \item $G + (x \rightarrow y)$: the DAG obtained by adding the edge $x \rightarrow y$ 
    to $G$. It is undefined if $x \rightarrow y$ is already in $G$ or if the resulting
    graph is not a DAG.
    \item $G - (x \rightarrow y)$: the DAG obtained by removing the edge $x \rightarrow
    y$. It is undefined if $x \rightarrow y$ is not in $G$.
    % use the counter-clockwise arrow for reversals
    \item  $G \circlearrowleft (x \rightarrow y)$: the DAG obtained by reversing the
    edge $x \rightarrow y$. It is undefined if $x \rightarrow y$ is not in $G$ or if the
    resulting graph is not a DAG.
    \item An edge is \textit{compelled} in a MEC $M$ if it is always directed in the
    same direction in all DAGs in $M$. By definition, the compelled edges of $M$ are
    exactly the directed edges of $M$'s canonical PDAG. 
    \item An edge is \textit{reversible} in a MEC $M$ if it is directed in one direction
    in some DAGs in $M$ and the other direction in other DAGs in $M$. By definition, the
    reversible edges of $M$ are exactly the undirected edges of $M$'s canonical PDAG.
\end{itemize}

\subsection{Navigating the Space of MECs}
GES explores the space of MECs by iteratively going from one MEC to the other, each time
selecting one with a higher score. GES defines a set of possible candidates that can be
reached from a given MEC. This defines the search space of GES
\citep{chickering2002optimal}.

For a MEC $M$, define:
\begin{itemize}
    \item Insertions: $\mathcal{I}(M) = \{ M' \in \mathcal{M} \mid \exists G \in M, \exists
    (x,y) \in [d]^2, G + (x \rightarrow y) \in M'\}$.
    \item Deletions: $\mathcal{D}(M) = \{M' \in \mathcal{M} \mid \exists G \in M, \exists (x,y)
    \in [d]^2, G - (x \rightarrow y) \in M'\}$.
    \item Reversals : $\mathcal{R}_c(M) = \{ M' \in \mathcal{M} \mid \exists G \in M, \exists (x,y)
    \in [d]^2, (x,y) \text{ is compelled in }P, G \circlearrowleft (x \rightarrow y) \in M'\}$.
    \item Reversals: $\mathcal{R}_r(M) = \{ M' \in \mathcal{M} \mid \exists G \in M, \exists (x,y)
    \in [d]^2, (x,y) \text{ is reversible in }P, G \circlearrowleft (x \rightarrow y) \in M'\}$.
\end{itemize}

The original GES algorithm first navigates through MECs only using $\mathcal{I}$, and
then only using $\mathcal{D}$. GIES proposes to add a last step and navigate through MECs using only
$\mathcal{R}_c \cup \mathcal{R}_r$ (after $\mathcal{I}$ and $\mathcal{D}$)
\citep{hauser2012characterization}.

In contrast, XGES navigates through MECs using simultaneously $\mathcal{I}$,
$\mathcal{D}$ and $\mathcal{R}_c$. The XGES heuristics favor using $\mathcal{D}$, then
$\mathcal{R}_c$, and finally $\mathcal{I}$ (see \Cref{alg:xges-0}).

\begin{remark}
    XGES could also use $\mathcal{R}_r$, we leave this for future work.
\end{remark}

\subsection{Theoretical Guarantees of GES}
\label{appendix:sec:theoretical_guarantees_ges}

We reformulate the results in \citep{chickering2002optimal}. Given a distribution $p^*$
faithful to a graph $G^*$ with MEC $M^*$. GES is \textit{correct} if it returns $M^*$,
the MEC of $G^*$. It is the MEC formed by all the DAGs faithful to $p^*$. 

\begin{theorem}[\protect{\citep[Lemma 9]{chickering2002optimal}}] \label{thm:chickering9}
    If no candidate MECs
    in $\mathcal{I}(M)$ can increase the score $S$, then $M$ has a special structure:
    all the independencies in $M$ are also independencies in $p^*$. We write it:
    $$\max_{M' \in \mathcal{I}(M)} S(M') \leq S(M) \Rightarrow \texttt{P1}(M; p^*),$$ where
    $\texttt{P1}(M; p^*)$ is the proposition that all the independencies in $M$ are also
    independencies in $p^*$.
\end{theorem}
Note that $p^*$ might have more independencies than $M$, i.e. $p^*$ is not necessarily
faithful to $M$. This is because $M$ might have superfluous edges. For example, the MEC
of the complete DAGs satisfies $\texttt{P1}(M; p^*)$.

\begin{theorem}[\protect{\citep[Lemma 10]{chickering2002optimal}}]
    \label{thm:chickering10a} If all the independencies in $M$ are also independencies
    in $p^*$, then the same is true for all the MECs in $\mathcal{D}(M)$ that have a
    higher score than $M$. We write it:
    $$\texttt{P1}(M; p^*) \Rightarrow (\forall M' \in \mathcal{D}(M), S(M') \geq S(M)
    \Rightarrow \texttt{P1}(M'; p^*)).$$
\end{theorem}

\begin{theorem}[\protect{\citep[Lemma 10]{chickering2002optimal}}] \label{thm:chickering10b}
    If all the
    independencies in $M$ are also independencies in $p^*$, and if no candidate MECs in
    $\mathcal{D}(M)$ can increase the score $S$, then $M = M^*$. We write
    it:
    $$\texttt{P1}(M; p^*) \wedge \left[\left(\max_{M' \in \mathcal{D}(M)} S(M')\right) \leq S(M)\right]
    \Rightarrow \texttt{P2}(M; p^*),$$ where $\texttt{P2}(M; p^*)$ is the proposition
    that $M = M^*$ or equivalently that $p^*$ is faithful to all the DAGs in $M$
\end{theorem}

With \Cref{thm:chickering9,thm:chickering10a,thm:chickering10b}, it follows that GES is
correct: it will reach a MEC $M$ at the end of phase 1 such that $\mathtt{P1}(M; p^*)$
is true. From there, all MECs visited in phase 2 will also have $\mathtt{P1}(M; p^*)$
true, and GES will stop on a MEC $M$ such that $\mathtt{P2}(M; p^*)$ is true, i.e. $M =
M^*$.

\subsection{Theoretical Guarantees of XGES-0 and XGES}
\subsubsection{Theoretical Guarantees of XGES-0}
\label{appendix:sec:theoretical_guarantees_xges0}
We now prove that XGES-0 is correct. 

\thXgesZero*
\begin{proof}We prove that XGES-0 terminates and is correct.
    \begin{itemize}
        \item \textbf{Termination.} The algorithm terminates because the search space is finite
        and the score is non-decreasing at each step.
        \item \textbf{Correctness.} Let $\hat M$ be the MEC returned by XGES-0. Since XGES-0
        terminated, it means that no candidate MECs in $\mathcal{I}(\hat M)$ can increase the
        score $S$. By \Cref{thm:chickering9}, $\texttt{P1}(\hat M; p^*)$ is true. It also
        means that no candidate MECs in $\mathcal{D}(\hat M)$ can increase the score $S$. By
        \Cref{thm:chickering10b}, $\texttt{P2}(\hat M; p^*)$ is true. Hence, XGES-0 is
        correct.
    \end{itemize}
\end{proof}

\clearpage
\subsubsection{Theoretical Guarantees of XGES}
\label{appendix:sec:theoretical_guarantees_xges}
We provide the pseudocode of XGES in \Cref{alg:xges}.

\begin{algorithm}[h]
    \LinesNumbered
    \DontPrintSemicolon
    \KwData{Data $\data \in \mathbb{R}^{n\times d}$, score function $S$}
    \KwDefine{$\delta_{\data,M}(O) = S(\text{Apply}(O,M) ; \data) - S(M; \data)$}  
    \KwResult{MEC of $G^*$}
    $M \leftarrow $ XGES-0$(X, S)$\;
    $\mathcal{D} \leftarrow$ all deletes valid for $M$\;
    \While{$|\mathcal{D}| > 0$}{
        $D^* \leftarrow \argmax_{D \in \mathcal{D}} \{\delta_{\data,M}(D)\}$\;
        % $\delta = s_{X,M}(D^*)$\;
        $M' \leftarrow$ Apply($D^*, M$)\;  
        $\widetilde{\mathcal{I}} \leftarrow$ all insertions, from any MEC to any MEC, that reinsert the edge deleted by
        $D^*$ \;
        
        $M' \leftarrow$ XGES-0\textsuperscript{*}$(X, S, M', \widetilde{\mathcal{I}})$\;  
        
        \tcc{XGES-0\textsuperscript{*} is a modified version of XGES from
        \Cref{alg:xges-0} that accepts an initial graph $M'$, and a set of forbidden
        insertions $\widetilde{\mathcal{I}}$}  
        \uIf{$S(M'; \data) > S(M; \data)$}{ $M \leftarrow M'$\; $\mathcal{D} \leftarrow$ all
        deletes valid for $M$\; }  
        \Else{ $\mathcal{D} \leftarrow \mathcal{D} \backslash \{D^*\}$\; } }  
        \Return $M$\;
    \caption{XGES}
    \label{alg:xges}
\end{algorithm}

\thXges*

\begin{proof}

    With a locally consistent score, XGES-0 is correct, so $M$ in line 1 is already the 
    MEC $M^*$ of $G^*$. Such a $M^*$ has the maximum score, so lines 9 and 10 are never
    executed. Hence, XGES is correct.

    For completeness, we provide proof of XGES's termination even when the score is not
    locally consistent. 

    \textbf{Termination in practice.} Every time XGES replaces $M$ by $M'$, the score of
        $M$ strictly increases. Since the search space is finite, XGES cannot infinitely
        replace $M$ by $M'$. But then, the other possibility is to remove $D^*$ from
        $\mathcal{D}$, which is a finite set. Hence, XGES terminates.

\end{proof}

\clearpage
\subsection{Navigating the Space of MECs with CPDAGs}
\label{appendix:sec:navigating_space_cpags}
We review how to translate the MEC formulation into a CPDAG formulation for the
practical implementations of GES and XGES. We recall that each MEC $M$ can be uniquely
represented by a CPDAG $P$.

\subsubsection{Original Parametrization}
\label{appendix:sec:ges_parametrization}
In \Cref{sec:manipulating_mecs} we reviewed that given a MEC $M$ represented by the CPDAG $P$,
each $M' \in \mathcal{I}(M)$ can be uniquely associated to an operator Insert$(x,y,T)$,
where $T$ is a set of nodes, such that applying Insert$(x,y,T)$ to $P$ yields a PDAG
$P'$ that represents $M'$ (up to completing it into a CPDAG). The operator
Insert$(x,y,T)$ modifies $P$ by adding the edge $x \rightarrow y$ and orienting all
edges $t - y$ as $t \rightarrow y$ for all $t \in T$. Not all Insert$(x,y,T)$ operators
can be applied to $P$ and yield a PDAG $P'$ that represents a MEC $M' \in
\mathcal{I}(M)$. All the Insert$(x,y,T)$ operators that correspond to a MEC $M' \in
\mathcal{I}(M)$ are called valid operators. There is a bijection between
$\mathcal{I}(M)$ and the set of valid Insert$(x,y,T)$, and there exist conditions that
can be checked on $P$ to determine whether Insert$(x,y,T)$ is valid or not. 

The same holds for $\mathcal{D}(M)$, $\mathcal{R}_c(M)$ and $\mathcal{R}_r(M)$. We
summarize the operators, their validity conditions, their score and their actions on a CPDAG in \Cref{tab:ges_operators_original},
adapted from \citet{chickering2002optimal} for insertion and deletion, and from
\citet{hauser2012characterization} for reversals.

\begin{table}[h!]
    \resizebox{\linewidth}{!}{
    \begin{tabular}{ C{1.3cm}C{5.5cm}C{5.8cm}C{5.5cm} }
        \toprule
        Operator & Insert$(x,y,T)$ & Delete$(x,y, H)$ & Reversal$(x,y,T)$ \\
        \midrule
        Conditions & 
        \begin{itemize}
            \item $x \not\in \Ad(y)$
            \item $T \subset \Ne(y) \backslash \Ad(x)$
            \item $[\Ne(y) \cap \Ad(x)] \cup T$ is a clique
            \item All semi-directed paths from $y$ to $x$ are blocked by $[\Ne(y) \cap
            \Ad(x)] \cup T$  
        \end{itemize} &
        \begin{itemize}
            \item $x \in \Ch(x) \cup \Ne(x)$
            \item $H \subset \Ne(y) \cap \Ad(x)$
            \item $[\Ne(y) \cap \Ad(x)] \backslash H$ is a clique
        \end{itemize} &
        \begin{itemize}
            \item $y \in \Pa(x)$
            \item $T \subset \Ne(y) \backslash \Ad(x)$
            \item $[\Ne(y) \cap \Ad(x)] \cup T$ is a clique
            \item All semi-directed paths from $y$ to $x$ other than ($y,x$) are blocked
            by $[\Ne(y) \cap \Ad(x)] \cup T\cup\Ne(x)$  
        \end{itemize}\\
        
        Score Increase &  
        $\squeezespaces{0}s(y, [\Ne(y) \cap \Ad(x)] \cup T \cup \Pa(y) \cup \{x\})$ $ -
        s(y, [\Ne(y) \cap \Ad(x)] \cup T \cup \Pa(y))$  
        & $s(y, [\Ne(y) \cap \Ad(x)] \backslash H \cup \Pa(y) \backslash \{x\}) - s(y,
        [\Ne(y) \cap \Ad(x)] \backslash H\cup \Pa(y)\cup\{x\})$  
        & $\squeezespaces{0}s(y, [\Ne(y) \cap \Ad(x)] \cup T \cup \Pa(y) \cup \{x\})$ $ -
        s(y, [\Ne(y) \cap \Ad(x)] \cup T \cup \Pa(y))$ 
        $+s(x, \Pa(x) \backslash \{y\}) - s(x, \Pa(x))$\\
        \\
        Actions & \begin{itemize}
            \item Add $x \rightarrow y$ to $P$.
            \item For all $t \in T$, orient $t - y$ as $t \rightarrow y$.
        \end{itemize} & \begin{itemize}
            \item Remove $x \rightarrow y$ (or $x - y$) from $P$.
            \item Orient all edges $y - h$ as $h \rightarrow y$ for $h \in H$.
            \item Orient all edges $x - h$ as $h \rightarrow x$ for $h \in H$.
        \end{itemize} & \begin{itemize}
            \item Reverse $x \rightarrow y$ into $x \leftarrow y$.
            \item For all $t \in T$, orient $t - y$ as $t \rightarrow y$.
        \end{itemize} 
        \\
        \bottomrule
    \end{tabular}
    } \caption{Summary of the operators and their conditions as described in
    \citep{chickering2002optimal}. The conditions for Insert and Delete are described in
    \citet[Definition 12, Definition 13, Theorem 15, Theorem 17, Table
    1]{chickering2002optimal}, and the score increase in \citet[Corollary 16, Corollary
    18]{chickering2002optimal}. The conditions for Reversal are described in
    \citet[Proposition 34]{hauser2012characterization}, and the score increase in
    \citet[Corollary 36]{hauser2012characterization}.}
    \label{tab:ges_operators_original}
\end{table}

\clearpage
\subsubsection{XGES Parametrization}
\label{appendix:sec:xges_parametrization}
We proposed a slightly different parametrization of the operators with adapted
conditions. We recall from \Cref{sec:updating_score} that with the original
parametrization, the scores of the operators depend on which PDAG they are applied
to. With the goal of an efficient implementation that caches the score of valid
operators to avoid recomputation, it is important and convenient for each operator to
have a unique score, agnostic of the PDAG it is applied to.

% To achieve this, we changed the parametrization of the operators to include the set of
% nodes used in the score computation as part of the operator, and add a validity
% condition about this set. 

We described the new parametrization for the Insert operator in \Cref{sec:updating_score}.
We describe the Delete and Reverse operator in \Cref{tab:ges_operators} hereafter.
For the Delete operator, we also replace the set $H$ by the set $C$, its complement in
$\Ne(y) \cap \Ad(x)$.

\begin{table*}[h!]
    \centering
    \resizebox{0.95\linewidth}{!}{
\begin{tabular}{ C{0.33\linewidth} C{0.32\linewidth} C{0.33\linewidth} }
\toprule
Insert($x, y, T$; $E$) & Delete($x, y, C; E$) &  Reverse($x, y, T, E, F$)\\
\midrule
\vspace{-0.4cm}
   \begin{itemize}
       \item[\textbf{I1}:] $y \not\in \Ad(x)$
       \item[\textbf{I2}:] $T \subset \Ne(y) \backslash \Ad(x)$
       \item[\textbf{I3}:] $(\Ne(y) \cap \Ad(x)) \cup T$ is a clique
       \item[\textbf{I4}:] All semi-directed paths from $y$ to $x$ have a node in $(\Ne(y) \cap
       \Ad(x)) \cup T$
       \item[\textbf{I5}:] $E = (\Ne(y) \cap \Ad(x)) \cup T \cup \Pa(y)$ 
   \end{itemize}
   & 
\vspace{-0.4cm}
   \begin{itemize}
       \item[\textbf{D1}:] $y \in \Ch(x) \cup \Ne(x)$
       \item[\textbf{D2}:] $C \subset \Ne(y) \cap \Ad(x)$
       \item[\textbf{D3}:] $C$ is a clique
       \item[\textbf{D4}:] $E = C \cup \Pa(y)$
   \end{itemize}
   & \vspace{-0.4cm}\begin{itemize}
        \item[\textbf{R1}:] $y \in \Pa(x)$ 
        \item[\textbf{R2}:] $T \subset \Ne(y) \backslash \Ad(x)$ 
        \item[\textbf{R3}:] $(\Ne(y) \cap \Ad(x)) \cup T$ is a clique 
        \item[\textbf{R4}:] All semi-directed paths from $y$ to $x$ not using edge $y \rightarrow x$
       have a node in $(\Ne(y) \cap \Ad(x)) \cup T \cup\Ne(x)$ 
       \item[\textbf{R5}:] $E = (\Ne(y) \cap \Ad(x)) \cup T \cup \Pa(y)$ 
       \item[\textbf{R6}:] $F = \Pa(x)$
    \end{itemize} \\%[-0.3cm]
    % \midrule
      $\delta = s(y, E \cup \{x\}) - s(y, E )$
    & $\delta=s(y, E \cup \{x\}) - s(y, E \backslash \{x\})$
    & $\delta= s(y, E \cup \{x\}) - s(y, E ) + s(x, F \backslash \{y\}) - s(x, F)$\\
\bottomrule
\end{tabular}}  
\caption{Parametrization of operators by XGES with their validity conditions and score. 
This parametrization renders the score of each operator invariant to the CPDAG it is 
applied to.}
\label{tab:ges_operators}
\end{table*}

\clearpage

\subsection{Efficient Algorithmic Formulation}
\label{appendix:sec:efficient_algorithmic_formulation}

We study how each validity condition described in \Cref{tab:ges_operators} can become 
true after each type of edge update in a PDAG. We only need to consider single-edge updates because 
of \Cref{thm:update_validity}.

\parhead{Edge Updates} There are seven types of edge updates: one of $a\quad b$, $a -
b$, or $a \rightarrow b$ becomes another one (6 = 3*2); and the reversal $a \rightarrow
b$ into $a \leftarrow b$ (which happens only after applying a reversal operator).

For each of these edge updates, we study how they affect the validity conditions of each operator. We 
summarize the results in \Cref{tab:operator_updates} and provide the detailed proofs in the following sections.
We further aggregate the results from \Cref{tab:operator_updates} into \Cref{tab:operator_updates2}.

\begin{table*}[h]
    \centering
    % resize box
    \resizebox{\textwidth}{!}{
    \begin{tabular}{ C{1cm}C{2.5cm}C{2.5cm}C{2.5cm}C{2.5cm}C{2.5cm}C{2.5cm}C{2.5cm} }
    \toprule
     & $a \quad b$ & $a \quad b$ & $a - b$ & $a - b$ & $a \rightarrow b$ & $a
     \rightarrow b$ & $a \rightarrow b$ \\
     & $a - b$ & $a \rightarrow b$ & $a \quad b$ & $a \rightarrow b$ & $a \quad b$ & $a
    - b$ & $a \leftarrow b$ \\
    \midrule
    \textbf{I1} & $\varnothing$ & $\varnothing$ & $\{a, b \} = \{x,y\}$  & $\varnothing$
    & $\{a, b \} = \{x,y\} $ & $\varnothing$ & $\varnothing$ \\[4mm]
    \textbf{I2} & $\curlystack{\bar{a} = y \\ \bar{b} \not\in \Ad(x)}$ & $\varnothing$ &
     $\curlystack{\bar{a}=x \\ \bar{b} \in \Ne(y)} $ & $\varnothing$ &
     $\curlystack{\bar{a}=x \\ \bar{b} \in \Ne(y)} $ & $\curlystack{\bar{a} = y \\
     \bar{b} \not\in \Ad(x)}$  & $\varnothing$\\[4mm]
    \textbf{I3}\textsuperscript{I2} & $\{a, b\} \subset \Ne(y)$ & $\{a, b\} \subset \Ne(y)$ &
    $\curlystack{\bar{a} = y \\ \bar{b} \in \Ad(x)}$ or $\curlystack{\bar{a} = x \\
    \bar{b} \in \Ne(y)}$ & $\curlystack{\bar{a} = y \\ \bar{b} \in \Ad(x) }$ &
    $\curlystack{\bar{a} = x \\ \bar{b} \in \Ne(y) }$ & $\varnothing$ &
    $\varnothing$\\[4mm]
    \textbf{I4} & $\curlystack{\bar{a} = y \\ \bar{b} \in \Ad(x)}$ or
    $\curlystack{\bar{a} = x \\ \bar{b} \in \Ne(y)}$ & $\curlystack{\bar{a} = x \\
    \bar{b} \in \Ne(y)}$ & \texttt{SD}($x,y; \bar{a},\bar{b}$) & \texttt{SD}($x,y; b,a$)
    & \texttt{SD}($x,y; a,b$) & $\curlystack{\bar{a} = y \\ \bar{b} \in \Ad(x)}$ &
    \texttt{SD}($x,y; a,b$) \\[4mm]
    \textbf{I5} & $\curlystack{\bar{a} = y \\ \bar{b} \in \Ad(x)}$ or
    $\curlystack{\bar{a} = x \\ \bar{b} \in \Ne(y)}$ & $b = y$ or $\curlystack{\bar{a} =
    x \\ \bar{b} \in \Ne(y)}$ & $\curlystack{\bar{a} = y \\ \bar{b} \in \Ad(x)}$ or
    $\curlystack{\bar{a} = x \\ \bar{b} \in \Ne(y)}$ & $b = y$ or $\curlystack{\bar{a} =
    y \\ \bar{b} \in \Ad(x)}$ & $b = y$ or $\curlystack{\bar{a} = x \\ \bar{b} \in
    \Ne(y)}$ & $b = y$ or $\curlystack{\bar{a} = y \\ \bar{b} \in \Ad(x)}$ & $\bar{a}=y$
    \\
    \midrule
    \textbf{D1} & $\{a, b \} = \{x,y\} $ & $(a,b) = (x,y)$ & $\varnothing$ &
    $\varnothing$ & $\varnothing$ & $(a,b) = (y,x)$ & $(a,b) = (y,x)$\\
    \textbf{D2} & $\curlystack{\bar{a}=y \\ \bar{b} \in \Ad(x)}$ or
    $\curlystack{\bar{a}=x \\ \bar{b} \in \Ne(y)}$ & $\curlystack{\bar{a}=x \\ \bar{b}
    \in \Ne(y)}$ & $\varnothing$ & $\varnothing$ & $\varnothing$ &
    $\curlystack{\bar{a}=y \\ \bar{b} \in \Ad(x)}$ & $\varnothing$\\
    \textbf{D3}\textsuperscript{D2} & $\{a,b\} \subset \Ne^u(y) \cap \Ad^u(x)$ & $\{a,b\} \subset \Ne^u(y)
    \cap \Ad^u(x)$ & $\varnothing$ & $\varnothing$ & $\varnothing$ & $\varnothing$ &
    $\varnothing$\\
    \textbf{D4} & $\varnothing$ & $b=y$ & $\varnothing$ & $b=y$ & $b=y$ & $b=y$ &
    $\bar{a} = y$\\
    \midrule
    \textbf{R1} & $\varnothing$ & $(a,b) = (y,x)$ & $\varnothing$ & $(a,b) = (y,x) $ &
    $\varnothing$ & $\varnothing$ & $(a,b) = (x,y)$\\
    \textbf{R2} & See \textbf{I2} & See \textbf{I2} & See \textbf{I2} & See \textbf{I2} &
    See \textbf{I2} & See \textbf{I2} & See \textbf{I2}\\
    \textbf{R3}\textsuperscript{R2} & See \textbf{I3} & See \textbf{I3} & See \textbf{I3} &
    See \textbf{I3} & See \textbf{I3} & See \textbf{I3} & See \textbf{I3}\\
    \textbf{R4} & See \textbf{I4} or $\bar{a}=x$ & See \textbf{I4} & See \textbf{I4} 
    & See \textbf{I4} & See \textbf{I4} & See \textbf{I4} or $x=\bar{a}$ & See \textbf{I4} \\
    \textbf{R5} & See \textbf{I5} & See \textbf{I5} & See \textbf{I5} & See \textbf{I5}
    & See \textbf{I5} & See \textbf{I5} & See \textbf{I5} \\
    \textbf{R6} & $\varnothing$ & $b=x$ & $\varnothing$ & $b=x$ & $b=x$ & $b=x$ &
    $\bar{a} = x$\\
    \bottomrule
    \end{tabular}
    } \caption{For each type of edge update involving an edge $(a,b)$, we list necessary
    conditions for each validity conditions of operators Insert$(x,y,T,E)$,
    Delete$(x,y,C,E)$, and Reverse$(x,y,T,E,F)$ to become valid. The notation 
    Condition($\bar a, \bar b$) is a shorthand for $\text{Condition}(a,b) \vee \text{Condition}(b,a)$.    
    The notation \texttt{SD}($x,y; a,b$) is a shorthand for the necessary condition:
    $(a,b)$, in that order, is on a semi-directed path from $y$ to $x$. All operators
    $\Pa, \Ne, \Ad$ and \texttt{SD}($x,y; a,b$) are computed with respect to the PDAG
    before the edge update. All operators $\Pa^u, \Ne^u, \Ad^u$ are computed with
    respect to the PDAG after the edge update.}
    \label{tab:operator_updates}
\end{table*}

\Cref{tab:operator_updates2} rewrites \Cref{tab:operator_updates} with the statements to
be conditions centered around $x$ and $y$, and aggregate all the necessary conditions
together. Whenever a single edge $(a,b)$ is updated, only the Insert operators satisfying
the condition \textbf{I}-any can become valid and need to be checked. The same holds for
Delete with \textbf{D}-any, and Reverse with \textbf{R}-any.
\begin{table*}[h!]
    \centering
    % resize box
    \resizebox{\textwidth}{!}{
    \begin{tabular}{ cC{4.2cm}C{3.5cm}C{4.2cm}C{2.2cm}C{2.5cm}C{2.4cm}C{2.2cm} }
    \toprule
     & $a \quad b$ & $a \quad b$ & $a - b$ & $a - b$ & $a \rightarrow b$ & $a
     \rightarrow b$ & $a \rightarrow b$ \\
     & $a - b$ & $a \rightarrow b$ & $a \quad b$ & $a \rightarrow b$ & $a \quad b$ & $a
    - b$ & $a \leftarrow b$ \\
    \midrule
    \textbf{I1} & $\varnothing$ & $\varnothing$ & $\{x,y\} = \{a, b \}$  & $\varnothing$
    & $\{x,y\} = \{a, b \} $ & $\varnothing$ & $\varnothing$ \\[2mm]
    \textbf{I2} & $y = \bar{a}$ & $\varnothing$ & $\curlystack{x= \bar{a} \\ y \in
     \Ne(\bar{b})} $ & $\varnothing$ & $\curlystack{x= \bar{a} \\ y \in \Ne(\bar{b})} $
     & $y = \bar{a}$  & $\varnothing$\\[5mm]
     \textbf{I3}\textsuperscript{I2} & $y \in \Ne(a)\cap \Ne(b)$ & $y \in \Ne(a)\cap \Ne(b)$ &
     $\curlystack{y = \bar{a} \\ x \in \Ad(\bar{b})}$ or $\curlystack{x = \bar{a} \\ y
     \in \Ne(\bar{b})}$ & $\curlystack{y = \bar{a} \\ x \in \Ad(\bar{b}) }$ &
     $\curlystack{x = \bar{a} \\ y \in \Ne(\bar{b}) }$ & $\varnothing$ &
     $\varnothing$\\[2mm]
    \textbf{I4} & $\curlystack{y = \bar{a} \\ x \in \Ad(\bar{b})}$ or $\curlystack{x =
    \bar{a} \\ y \in \Ad(\bar{b})}$ & $\curlystack{x = \bar{a} \\
    y \in \Ne(\bar{b})}$ & \texttt{SD}($x,y; \bar{a},\bar{b}$) & \texttt{SD}($x,y; b,a$)
    & \texttt{SD}($x,y; a,b$) & $\curlystack{y = \bar{a} \\ x \in \Ad(\bar{b})}$ &
    \texttt{SD}($x,y; a,b$) \\[5mm]
    \textbf{I5} & $\curlystack{y = \bar{a} \\ x \in \Ad(\bar{b})}$ or $\curlystack{x =
    \bar{a} \\ y \in \Ne(\bar{b})}$ & $y = b$ or $\curlystack{x = \bar{a} \\ y \in
    \Ne(\bar{b})}$ & $\curlystack{y = \bar{a} \\ x \in \Ad(\bar{b})}$ or $\curlystack{x
    = \bar{a} \\ y \in \Ne(\bar{b})}$ & $y = b$ or $\curlystack{y = \bar{a} \\ x \in
    \Ad(\bar{b})}$ & $y = b$ or $\curlystack{x = \bar{a} \\ y \in \Ne(\bar{b})}$ & $y =
    b$ or $\curlystack{y = \bar{a} \\ x \in \Ad(\bar{b})}$ & $y=\bar{a}$ \\
    \midrule
    \textbf{I}-any & 
    \begin{itemize}[leftmargin=*]
        \item $y \in \{a,b\}$ 
        \item $y \in \Ne(a)\cap \Ne(b)$
        \item $x = a$ and $y \in \Ne(b)$
        \item $x = b$ and $y \in \Ne(a)$
    \end{itemize}
    & \begin{itemize}[leftmargin=*]
        \item $y = b$
        \item $y \in \Ne(a)\cap \Ne(b)$
        \item $x = a$ and $y \in \Ne(b)$
        \item $x = b$ and $y \in \Ne(a)$
    \end{itemize}
    & \begin{itemize}[leftmargin=*]
        \item $x = a$ and $y \in \Ne(b) \cup \{b\}$
        \item $x = b$ and $y \in \Ne(a) \cup \{a\}$
        \item $y = a$ and $x \in \Ad(b)$
        \item $y = b$ and $x \in \Ad(a)$
        \item \texttt{SD}($x,y; a,b$)
        \item \texttt{SD}($x,y; b,a$)
    \end{itemize}
    & \begin{itemize}[leftmargin=*]
        \item $y = a$ and $x \in \Ad(b)$
        \item $y = b$
        \item \texttt{SD}($x,y; b,a$)
    \end{itemize}
    & \begin{itemize}[leftmargin=*]
        \item $y=b$
        \item $x = a$ and $y \in \Ne(b) \cup \{b\}$
        \item $x = b$ and $y \in \Ne(a) \cup \{a\}$
        \item \texttt{SD}($x,y; a,b$)
    \end{itemize}
    & \begin{itemize}[leftmargin=*]
        \item $y \in \{a,b\}$
    \end{itemize}
    & \begin{itemize}[leftmargin=*]
        \item $y \in \{a,b\}$
        \item \texttt{SD}($x,y; a,b$)
    \end{itemize}
    \\
    \midrule
    \textbf{D1} & $\{x,y\} = \{a, b \} $ & $(x,y) = (a,b)$ & $\varnothing$ &
    $\varnothing$ & $\varnothing$ & $(x,y) = (b,a )$ & $(x,y) = (b,a)$\\
    \textbf{D2} & $\curlystack{y = \bar{a} \\ x \in \Ad(\bar{b})}$ or
    $\curlystack{x = \bar{a} \\ y \in \Ne(\bar{b})}$ & $\curlystack{x = \bar{a} \\ 
    y \in \Ne(\bar{b})}$ & $\varnothing$ & $\varnothing$ & $\varnothing$ &
    $\curlystack{y = \bar{a} \\ x \in \Ad(\bar{b})}$ & $\varnothing$\\
    \textbf{D3}\textsuperscript{D2} & $\curlystack{x \in \Ad^u(a) \cap \Ad^u(b) \\ y
        \in \Ne^u(a) \cap \Ne^u(b)}$ & $\curlystack{ x \in \Ad^u(a) \cap \Ad^u(b) \\ y
        \in \Ne^u(a) \cap \Ne^u(b)}$ & $\varnothing$ & $\varnothing$ & $\varnothing$ &
        $\varnothing$ & $\varnothing$\\
    \textbf{D4} & $\varnothing$ & $y=b$ & $\varnothing$ & $y=b$ & $y=b$ & $y=b$ &
    $y\in\{a,b\}$\\
    \midrule
    \textbf{D}-any &
    \begin{itemize}[leftmargin=*]
        % \item $y = a$ and $x \in \Ad(b)$
        % \item $y = b$ and $x \in \Ad(a)$
        % \item $x = a$ and $y \in \Ne(b) \cup \{b\}$
        % \item $x = b$ and $y \in \Ne(a) \cup \{a\}$
        \item $y \in \{a,b\}$
        \item $x \in \{a,b\}$
        \item $x \in \Ad(a) \cap \Ad(b)$ and $y \in \Ne(a) \cap \Ne(b)$
    \end{itemize}
    & \begin{itemize}[leftmargin=*]
        % \item $x = a$ and $y \in \Ne(b)$
        % \item $x = b$ and $y \in \Ne(a)$
        \item $y = b$
        \item $x \in \{a,b\}$
        \item $x \in \Ad(a) \cap \Ad(b)$ and $y \in \Ne(a) \cap \Ne(b)$
    \end{itemize}
    & $\varnothing$
    & \begin{itemize}[leftmargin=*]
        \item $y = b$
    \end{itemize}
    & \begin{itemize}[leftmargin=*]
        \item $y = b$
    \end{itemize}
    & \begin{itemize}[leftmargin=*]
        \item $y \in \{a, b\}$ % and $x \in \Ad(b) \cup \{b\}$
        % \item $y = b$
    \end{itemize}
    & \begin{itemize}[leftmargin=*]
        \item $y \in \{a,b\}$
    \end{itemize}\\
    \midrule
    \textbf{R}-any & 
    \begin{itemize}[leftmargin=*]
        \item $y \in \{a,b\}$ 
        \item $y \in \Ne(a)\cap \Ne(b)$
        \item $x \in \{a,b\}$
    \end{itemize}
    & \begin{itemize}[leftmargin=*]
        \item $y = b$
        \item $y \in \Ne(a)\cap \Ne(b)$
        \item $x = a$ and $y \in \Ne(b)$
        \item $x = b$
    \end{itemize}
    & \begin{itemize}[leftmargin=*]
        \item $x = a$ and $y \in \Ne(b) \cup \{b\}$
        \item $x = b$ and $y \in \Ne(a) \cup \{a\}$
        \item $y = a$ and $x \in \Ad(b)$
        \item $y = b$ and $x \in \Ad(a)$
        \item \texttt{SD}($x,y; a,b$)
        \item \texttt{SD}($x,y; b,a$)
    \end{itemize}
    & \begin{itemize}[leftmargin=*]
        \item $y = a$ and $x \in \Ad(b)$
        \item $y = b$
        \item \texttt{SD}($x,y; b,a$)
        \item $x = b$
    \end{itemize}
    & \begin{itemize}[leftmargin=*]
        \item $y=b$
        \item $x = a$ and $y \in \Ne(b) \cup \{b\}$
        \item $x = b$
        \item \texttt{SD}($x,y; a,b$)
    \end{itemize}
    & \begin{itemize}[leftmargin=*]
        \item $y \in \{a,b\}$
        \item $x=b$
    \end{itemize}
    & \begin{itemize}[leftmargin=*]
        \item $y \in \{a,b\}$
        \item $x \in \{a,b\}$
        \item \texttt{SD}($x,y; a,b$)
    \end{itemize}
    \\
    \bottomrule
    \end{tabular}
    } 
    \caption{For each type of edge update involving an edge $(a,b)$, we list necessary
    conditions for each validity conditions of operators Insert$(x,y,T,E)$,
    Delete$(x,y,C,E)$, and Reverse$(x,y,T,E,F)$ to become valid.    
    The notation \texttt{SD}($x,y; a,b$) is a shorthand for the necessary condition:
    $(a,b)$, in that order, is on a semi-directed path from $y$ to $x$. All operators
    $\Pa, \Ne, \Ad$ and \texttt{SD}($x,y; a,b$) are computed with respect to the PDAG
    before the edge update. The rows \textbf{I}-any, \textbf{D}-any, and \textbf{R}-any
    aggregate the necessary conditions for each validity condition and express them in 
    a disjunctive form: at least one of the conditions must be true for the operator to
    become valid.}
    \label{tab:operator_updates2}
\end{table*}

\newtheorem{opup}{Operator Update}
\newtheorem{lemma}{Lemma}

\subsubsection{I1}

\begin{opup}[Updates on I1]
    Assume $\mathbf{I1}(x,y, T; E)$ is false and becomes true after an update involving
    $(a,b)$. Then,
    \begin{itemize}
        \item the update cannot be $U1: (a \quad b) \rightsquigarrow (a - b)$,
        \item the update cannot be $U2: (a \quad b) \rightsquigarrow (a \rightarrow b)$,
        \item if the update is $U3: (a - b) \rightsquigarrow (a \quad b)$ then $ \{a,b\}
        =\{x,y\}$,
        \item the update cannot be $U4: (a - b) \rightsquigarrow (a \rightarrow b)$,
        \item if the update is $U5: (a \rightarrow b) \rightsquigarrow (a \quad b)$ then
        $\{a,b\} =\{x,y\}$,
        \item the update cannot be $U6: (a \rightarrow b) \rightsquigarrow (a - b)$,
        \item the update cannot be $U7: (a \rightarrow b) \rightsquigarrow (a \leftarrow
        b)$.
    \end{itemize}
\end{opup}

\begin{proof}
    We recall that $\mathbf{I1}(x,y, T; E)$ is $y \not\in \Ad(x)$. So the assumptions
    are $y \in \Ad(x)$ and $y \not\in \Ad^u(x)$. But U1, U2, U4, U6, and U7 do not
    remove any elements from any $\Ad(x')$ set. So none of them can render
    $\mathbf{I1}(x,y, T; E)$ true. 

    U3 and U5 can only remove elements from $\Ad(a)$ or $\Ad(b)$, and do so only by
    removing $b$ or $a$, respectively. So $\{x,y\} = \{a,b\}$.
\end{proof}

\subsubsection{I2}
We start with a general lemma for I2.

\begin{lemma}[I2 to become true]\label{lemma:I2} Assume $\mathbf{I2}(x,y, T; E)$ is
    false and becomes true after an edge update. Then (i) $\Ne^u(y)$ gained an element,
    or (ii) $\Ad^u(x)$ lost an element that was in $\Ne(y)$.
\end{lemma}
\begin{proof}
    We recall that $\mathbf{I2}(x,y, T; E)$ is $T \subset \Ne(y) \backslash \Ad(x)$. If
     $\mathbf{I2}(x, y, T; E)$ changes from false to true then there exists $t \in T$
     that was not in $\Ne(y) \backslash \Ad(x)$ and is now in $\Ne^u(y) \backslash
     \Ad^u(x)$, which writes $$(t\not \in \Ne(y) \vee t \in \Ad(x)) \wedge t \in
     \Ne^u(y) \wedge t \not\in \Ad^u(x).$$
     \begin{itemize}
        \item If $t \in \Ne(y)$ then we must have $t \in \Ad(x)$ and $t \not\in
        \Ad^u(x)$. So $\Ad^u(x)$ lost $t$, which was in $\Ne(y) \cap \Ad(x)$.
        \item If $t \not\in \Ne(y)$, then $\Ne^u(y)$ gained $t$. 
     \end{itemize}
    In conclusion, either $\Ne^u(y)$ gained an element, or $\Ad^u(x)$ lost an element
     that was is in $\Ne(y)$. %(With more work and assuming \textbf{I1} holds after the
    %  update, we can also show that the element gained by $\Ne^u(y)$ is not in $\Ad(x)$).

\end{proof}

\begin{opup}[Updates on I2]
    Assume $\mathbf{I2}(x,y, T; E)$ is false and becomes true after an update involving
    $(a,b)$. Then,
    \begin{itemize}
        \item if the update is $U1: (a \quad b) \rightsquigarrow (a - b)$ then $y \in
        \{a, b\}$,
        \item the update cannot be $U2: (a \quad b) \rightsquigarrow (a \rightarrow b)$,
        \item if the update is $U3: (a - b) \rightsquigarrow (a \quad b)$ then
            $\curlystack{ a=x \\ b \in \Ne(y)}$ or $\curlystack{b=x \\ a \in \Ne(y)}$,
        \item the update cannot be $U4: (a - b) \rightsquigarrow (a \rightarrow b)$,
        \item if the update is $U5: (a \rightarrow b) \rightsquigarrow (a \quad b)$ then
            $\curlystack{ a=x \\ b \in \Ne(y)}$ or $\curlystack{b=x \\ a \in \Ne(y)}$,
        \item if the update is $U6: (a \rightarrow b) \rightsquigarrow (a - b)$ then $y
        \in \{a, b\}$,
        \item the update cannot be $U7: (a \rightarrow b) \rightsquigarrow (a \leftarrow
        b)$.
    \end{itemize}
\end{opup}

\begin{proof}
    According to \Cref{lemma:I2}, either $\Ne^u(y)$ gained an element or $\Ad^u(x)$ lost
    an element (that was in $\Ne(y)$).

    We now study the necessary conditions for each update, if it was applied and made
    $\mathbf{I2}(x,y, T; E)$ become true.
    \begin{itemize}
        \item The updates U1 and U6 can only add elements to $\Ne^u(a)$ or $\Ne^u(b)$
        and not remove any element to any $\Ad^u(x')$. So $\Ne^u(y)$ gained an element
        and $y \in \{a, b\}$.
        \item The updates U2, U4, and U7 do not add any elements to any $\Ne^u(y')$, and
        do not remove any elements to any $\Ad^u(x')$. So none of them can make
        $\mathbf{I2}(x,y, T; E)$ become true.
        \item The updates U3 and U5 can only remove elements from $\Ad^u(a)$ or
        $\Ad^u(b)$ and not add any element to any $\Ne^u(y')$. So $\Ad^u(x)$ lost an
        element and $x \in \{a, b\}$. If $x = a$ (resp. $x = b$) then the lost element
        must be $b$ (resp. $a$) and so $b \in \Ne(y)$ (resp. $a \in \Ne(y)$).
    \end{itemize}
\end{proof}

\subsubsection{I3}
We start with a general lemma for I3.
\begin{lemma}[I3 to become true]\label{lemma:I3} Assume $\mathbf{I3}(x,y, T; E)$ is
    false and becomes true after an edge update about $(a,b)$. Further, assume that
    $\mathbf{I2}(x,y, T; E)$ is true after the update (regardless of its status before
    the update). Then either (i) $\{a,b\} \subset \Ne(y)$ and the update rendered $a,b$
    adjacent, or (ii) $\Ne^u(y)$ lost an element that was in $\Ne(y) \cap \Ad(x)$, or
    (iii) $\Ad^u(x)$ lost an element that was in $\Ne(y) \cap \Ad(x)$.
\end{lemma}
\begin{proof}
    We recall that $\mathbf{I3}(x,y, T; E)$ is $[\Ne(y) \cap \Ad(x)] \cup T$ is a
    clique, and that $\mathbf{I2}(x,y, T; E)$ is $T \subset \Ne(y) \backslash \Ad(x)$.
    So the assumptions are $[\Ne(y) \cap \Ad(x)] \cup T$ is not a clique (in the
    pre-update PDAG) meanwhile $[\Ne^u(y) \cap \Ad^u(x)] \cup T$ is a clique (in the
    post-update PDAG), and $T \subset \Ne^u(y) \backslash \Ad^u(x)$.

    Since $[\Ne(y) \cap \Ad(x)] \cup T$ is not a clique, it must contain two nodes $c,d$
    that are not connected in the pre-update PDAG. 

    We distinguish two cases: 
    \begin{itemize}
        \item If $\{c,d\} \subset [\Ne^u(y) \cap \Ad^u(x)] \cup T$, then the update must
        have connected $c$ and $d$. So $c$ and $d$ are $a$ and $b$. Also since $T
        \subset \Ne^u(y) \backslash \Ad^u(x)$, then $\{a,b\} \subset [\Ne^u(y) \cap
        \Ad^u(x)] \cup T \subset \Ne^u(y)$. Finally, an update can only change one edge
        at a time, so $\Ne^u(y) = \Ne(y)$ (since $a$ and $b$ are not $y$ as $y$ cannot
        be a neighbor of itself). Hence, $\{a,b\} \subset \Ne(y)$ and $a$ and $b$ became
        adjacent.
        
        \item Else, $c$ or $d$ has been removed from $[\Ne^u(y) \cap \Ad^u(x)] \cup T$
        during the update. Without loss of generality, assume $c$ was removed. Since $T$
        does not change, then $c$ was removed from $[\Ne(y) \cap \Ad(x)]$. So $\Ne^u(y)$
        or $\Ad^u(x)$ lost an element that was in $\Ne(y) \cap \Ad(x)$.
    \end{itemize}
    
    Hence, [$\{a,b\} \subset \Ne(y)$ and $a$ and $b$ became adjacent], or $\Ne(y)$ lost
     an element that was in $\Ne(y) \cap \Ad(x)$ or $\Ad(x)$ lost an element that was in
     $\Ne(y) \cap \Ad(x)$.
\end{proof}

\begin{opup}[Updates on I3]
    Assume $\mathbf{I3}(x,y, T; E)$ is false and becomes true after an update involving
    $(a,b)$. Further, assume that
    $\mathbf{I2}(x,y, T; E)$ is true after the update. Then,
    \begin{itemize}
        \item if the update is $U1: (a \quad b) \rightsquigarrow (a - b)$ then $\{a,b\}
        \subset \Ne(y)$,
        \item if the update is $U2: (a \quad b) \rightsquigarrow (a \rightarrow b)$ then
        $\{a,b\} \subset \Ne(y)$,
        \item if the update is $U3: (a - b) \rightsquigarrow (a \quad b)$ then
        $\curlystack{ a \in  \{x,y\} \\ 
            b \in \Ne(y) \cap \Ad(x) }$ or $\curlystack{ b \in  \{x,y\} \\ 
            a \in \Ne(y) \cap \Ad(x) }$,
        \item if the update is $U4: (a - b) \rightsquigarrow (a \rightarrow b)$ then
        $\curlystack{ a =y \\ 
            b \in \Ne(y) \cap \Ad(x) }$ or $\curlystack{ b = y \\ 
            a \in \Ne(y) \cap \Ad(x) }$,
        \item if the update is $U5: (a \rightarrow b) \rightsquigarrow (a \quad b)$ then
        $\curlystack{ a = x \\ 
            b \in \Ne(y) \cap \Ad(x) }$ or $\curlystack{ b = x \\ 
            a \in \Ne(y) \cap \Ad(x) }$,
        \item if the update is $U6: (a \rightarrow b) \rightsquigarrow (a - b)$ then it
        is impossible,
        \item if the update is $U7: (a \rightarrow b) \rightsquigarrow (a \leftarrow b)$
        then it is impossible.
    \end{itemize}
    
\end{opup}

\begin{proof}
    According to \Cref{lemma:I3}, either $\{a,b\} \subset \Ne(y)$ and the update
    rendered $a,b$ adjacent, or $\Ne^u(y)$ lost an element that was in $\Ne(y) \cap
    \Ad(x)$, or $\Ad^u(x)$ lost an element that was in $\Ne(y) \cap \Ad(x)$.

    We now study the necessary conditions for each update, if it was applied and made
    $\mathbf{I3}(x,y, T; E)$ become true.
    \begin{itemize}
        \item The updates U1 and U2 can only add elements to sets like $\Ne^u(y)$ or
        $\Ad^u(x)$, so the only possibility is that $\{a,b\} \subset \Ne(y)$.
        \item The update U3 does not render any edge adjacent. So by \Cref{lemma:I3},
        $\Ne(y)$ or $\Ad(x)$ lost an element $c$ that was in $\Ne(y) \cap \Ad(x)$. 
        \begin{itemize}
            \item If it is $\Ne(y)$ that lost $c$, then we have $\{a,b\} = \{y,c\}$.
            Without loss of generality, $a = y$ and $b = c$, so $b \in \Ne(y) \cap
            \Ad(x)$. 
            \item If it is $\Ad(x)$ that lost $c$, then we have $\{a,b\} = \{x,c\}$.
            Without loss of generality, $a = x$ and $b = c$, so $b \in \Ne(y) \cap
            \Ad(x)$.
        \end{itemize} 
            So gathering all cases and with generality:
            $$ \curlystack{ a \in  \{x,y\} \\ 
                b \in \Ne(y) \cap \Ad(x) } \text{ or } \curlystack{ b \in  \{x,y\} \\ 
                a \in \Ne(y) \cap \Ad(x) }$$
        \item The update U4 does not render any edge adjacent, does not remove any
        element from any $\Ad(x')$, but removes elements from $\Ne^u(a)$ or $\Ne^u(b)$
        (resp.$b$ or $a$). So by \Cref{lemma:I3}, $a=y$ (resp. $b=y$) and $b \in \Ne(y)
        \cap \Ad(x)$ (resp. $a \in \Ne(y) \cap \Ad(x)$).
        \item The update U5 does not render any edge adjacent, does not remove any
        element from any $\Ne(y')$, but removes elements from $\Ad^u(a)$ or $\Ad^u(b)$
        (resp.$b$ or $a$). So by \Cref{lemma:I3}, $a=x$ (resp. $b=x$) and $b \in \Ne(y)
        \cap \Ad(x)$ (resp. $a \in \Ne(y) \cap \Ad(x)$).
        \item The updates U6 and U7 cannot remove any element from any $\Ne(y')$ or
        $\Ad(x')$, and do not make $a$ and $b$ adjacent (they were already adjacent). So
        by \Cref{lemma:I3}, they cannot make $\mathbf{I3}(x,y, T; E)$ become true.
\end{itemize}
\end{proof}

\subsubsection{I4}
We start with a general lemma for I4.
\begin{lemma}[I4 to become true]\label{lemma:I4} Assume $\mathbf{I4}(x,y, T; E)$ is
    false and becomes true after an edge update about $(a,b)$. If the update does not
    reverse a directed edge, does not direct an undirected edge, and does not delete an
    edge, then the update must have added an element to $[\Ne(y) \cap \Ad(x)]$.
    
    Otherwise, the update invalidated an edge on a semi-directed path from $y$ to $x$
    (where invalidated means that the edge $(a,b)$ cannot be traversed from $a$ to $b$
    anymore with the semi-directed rules: either $a$ and $b$ are not adjacent anymore,
    or the edge is now $a \leftarrow b$.).
\end{lemma}

\begin{proof}
    We recall that $\mathbf{I4}(x,y, T; E)$ is: all semi-directed paths from $y$ to $x$
    have a node in $[\Ne(y) \cap \Ad(x)] \cup T$. If the condition does not hold before
    the update, then there exists a semi-directed path from $y$ to $x$ with no node in
    $[\Ne(y) \cap \Ad(x)] \cup T$. We distinguish two cases:

    If the update does not remove or reverse any edge, then the semi-directed path is
    still there after the update. For the condition to become true, the update must have
    added an element to $[\Ne(y) \cap \Ad(x)] \cup T$ (one element that is on the
    semi-directed path).

    Since $T$ does not change, the update must have added an element to $[\Ne(y) \cap
    \Ad(x)]$.

    Otherwise, the semi-directed path from $y$ to $x$ is not a semi-directed path
    anymore. So the update invalidated an edge on it: either $a$ and $b$ are not
    adjacent anymore, or the edge is now $a \leftarrow b$.
\end{proof}

\begin{opup}[Updates on I4]
    \label{opup:updates_on_I4}
    Assume $\mathbf{I4}(x,y, T; E)$ is false and becomes true after an update involving
    $(a,b)$. Then,
    \begin{itemize}
        \item if the update is $U1: (a \quad b) \rightsquigarrow (a - b)$ then
        $\curlystack{b=y \\ a \in \Ad(x)}$ or $\curlystack{b=x \\ a \in \Ne(y)}$ or
        $\curlystack{a=y \\ b \in \Ad(x)}$ or $\curlystack{a=x \\ b \in \Ne(y)}$,
        \item if the update is $U2: (a \quad b) \rightsquigarrow (a \rightarrow b)$ then
        $\curlystack{b=x \\ a \in \Ne(y)}$ or $\curlystack{a=x \\ b \in \Ne(y)}$.
        % \item if the update is $U3: (a - b) \rightsquigarrow (a \quad b)$ or $U4: (a -
        % b) \rightsquigarrow (a \rightarrow b)$ or $U5: (a \rightarrow b)
        % \rightsquigarrow (a \quad b)$ or $U7: (a \rightarrow b) \rightsquigarrow (a
        % \leftarrow b)$         
        %  then $a - b$ (or $a \rightarrow b$) is in on a semi-directed path from $y$ to
        % $x$ . $\curlystack{b=x \\ a \in \Ne(y)}$ or $\curlystack{a=x \\ b \in
        % \Ne(y)}$,
        \item if the update is $U3: (a - b) \rightsquigarrow (a \quad b)$ then either
        $(a,b)$ or $(b,a)$ was on a semi-directed path from $y$ to $x$.

        \item if the update is $U4: (a - b) \rightsquigarrow (a \rightarrow b)$ then
        $(b,a)$ was on a semi-directed path from $y$ to $x$.
        \item if the update is $U5: (a \rightarrow b) \rightsquigarrow (a \quad b)$ then
        $(a,b)$ was on a semi-directed path from $y$ to $x$.
        \item if the update is $U6: (a \rightarrow b) \rightsquigarrow (a - b)$ then
        $\curlystack{b=y \\ a \in \Ad(x)}$ or $\curlystack{a=y \\ b \in \Ad(x)}$.
        \item if the update is $U7: (a \rightarrow b) \rightsquigarrow (a \leftarrow b)$
        then $(a,b)$ was on a semi-directed path from $y$ to $x$.
    \end{itemize}
\end{opup}

\begin{proof}
    Assume $\mathbf{I4}(x,y, T; E)$ is false and becomes true after an update involving
    $(a,b)$. U1, U2, and U6 do not reverse any directed edge, do not direct any
    undirected edge, and do not delete any edge. So by \Cref{lemma:I4}, these updates
    must have added an element to $[\Ne(y) \cap \Ad(x)]$. Without loss of generality for
    now, assume $a$ was added. Notice that $a$ cannot have been added to both $\Ne(y)$
    and $\Ad(x)$ (otherwise $y=b$ and $x=b$, yet $x\neq y$), so $a$ was already in one
    of them before the update. 
    \begin{itemize}
        \item If the update is U1 then $a$ can have been added to $\Ne(y)$ or $\Ad(x)$,
        and already present in the other one. Hence we have, in full generality,
        $\curlystack{b=y \\ a \in \Ad(x)}$ or $\curlystack{b=x \\ a \in \Ne(y)}$ or
        $\curlystack{a=y \\ b \in \Ad(x)}$ or $\curlystack{a=x \\ b \in \Ne(y)}$.
        \item If the update is U2 then $a$ can only have been added to $\Ad(x)$, and so
        already present in $\Ne(y)$. Hence we have, in full generality, $\curlystack{b=x
        \\ a \in \Ne(y)}$ or $\curlystack{a=x \\ b \in \Ne(y)}$.
        \item If the update is U6 then $a$ can only have been added to $\Ne(y)$, and so
        already present in $\Ad(x)$. Hence we have, in full generality, $\curlystack{b=y
        \\ a \in \Ad(x)}$ or $\curlystack{a=y \\ b \in \Ad(x)}$.
    \end{itemize}

    U3, U4, U5, and U7 cannot add any element to $[\Ne(y) \cap \Ad(x)]$. So by
    \Cref{lemma:I4}, these updates must have invalidated an edge on a semi-directed path
    from $y$ to $x$. 
    \begin{itemize}
        \item If the update is U3 then either $(a,b)$ or $(b,a)$ was on a semi-directed
        path from $y$ to $x$.
        \item If the update is U4 then $(b,a)$ was on a semi-directed path from $y$ to
        $x$.
        \item If the update is U5 or U7, then $(a,b)$ was on a semi-directed path from
        $y$ to $x$.        
    \end{itemize}
\end{proof}

So far, all the necessary conditions for the updates were efficient to test, e.g. finding
all the insert with $y \in \{a, b\}$.

With $\mathbf{I4}$ however, we have the condition that $(a,b)$ or $(b,a)$ was on a
semi-directed path from $y$ to $x$. This can be inefficient to test. 
A speed-up can be obtained by proceeding as follows
\begin{itemize}
    \item Instead of ensuring that $\mathcal{C}$ always contains all valid Insert
    operators (which means all operators that with $\mathbf{I1}$, $\mathbf{I2}$,
    $\mathbf{I3}$, $\mathbf{I4}$, and $\mathbf{I5}$ are true), we can ensure that
    $\mathcal{C}$ always contains all Insert operator for which $\mathbf{I1}$,
    $\mathbf{I2}$, $\mathbf{I3}$, and $\mathbf{I5}$ are true. 
    \item Recall that at each step, XGES involves 4 substeps described in
    \Cref{sec:efficient_algorithmic_formulation}. Substep 3 verifies that the operator
    is valid. If we notice an operator that is invalid because of $\mathbf{I4}$, then we
    can put it aside. Additionally, we save the path from $y$ to $x$ that was rendering 
    the operator invalid. 
    \item The new necessary condition for the operator to be valid is that an edge on
    the saved path gets removed (or blocked). Whenever we remove or block an edge from
    the path, we can re-verify the operator.
\end{itemize}

\subsubsection{I5}
We start with a general lemma for I5.

\begin{lemma}[I5 to become true]\label{lemma:I5} Assume $\mathbf{I5}(x,y, T; E)$ is
    false and becomes true after an edge update about $(a,b)$. Then (i) $\Pa(y)$
    changed, or (ii) $[\Ne(y) \cap \Ad(x)]$ changed.

    Condition (ii) can be further broken down into: (ii.a) $\Ne^u(y)$ lost an element
    that is in $\Ad(x)$, or (ii.b) $\Ne^u(y)$ gained an element that is in $\Ad(x)$, or
    (ii.c) $\Ne^u(y)$ gained an element that is in $\Ad(x)$, or (ii.d) $\Ad^u(x)$ gained
    an element that is in $\Ne(y)$.
\end{lemma}

\begin{proof}
    Assume $\mathbf{I5}(x,y, T; E)$ is false and becomes true after an edge update about
    $(a,b)$. Since $E$ and $T$ do not change, we must have $[\Ne^u(y) \cap \Ad^u(x)]
    \cup \Pa^u(y) \neq [\Ne(y) \cap \Ad(x)] \cup \Pa(y)$. Either $\Pa(y)$ changed or
    $[\Ne(y) \cap \Ad(x)]$ changed.

    Assume $[\Ne(y) \cap \Ad(x)]$ changed. If it lost an element $c \in \Ne(y) \cap
    \Ad(x)$, then $c$ was removed from $\Ne^u(y)$ or $\Ad^u(x)$. If it gained an element
    $c \not\in \Ne(y) \cap \Ad(x)$, then $c$ was added to $\Ne^u(y)$ or $\Ad^u(x)$.
    Since $c \in \Ne^u(y) \cap \Ad^u(x)$ and only one of $\Ne^u(y)$ and $\Ad^u(x)$ can
    change at a time (see proof \Cref{opup:updates_on_I4}), then either $c$ was added to
    $\Ne^u(y)$ and $c \in \Ad(x)$, or $c$ was added to $\Ad^u(x)$ and $c \in \Ne(y)$.
\end{proof}

\begin{opup}[Updates on I5]
    Assume $\mathbf{I5}(x,y, T; E)$ is false and becomes true after an update involving
    $(a,b)$. Then,
    \begin{itemize}
        \item if the update is $U1: (a \quad b) \rightsquigarrow (a - b)$ then
        $\curlystack{b=y \\ a \in \Ad(x)}$ or $\curlystack{b=x \\ a \in \Ne(y)}$ or
        $\curlystack{a=y \\ b \in \Ad(x)}$ or $\curlystack{a=x \\ b \in \Ad(y)}$,
        \item if the update is $U2: (a \quad b) \rightsquigarrow (a \rightarrow b)$ then
         $y = b$ or $\curlystack{b=x \\ a \in \Ne(y)}$ or $\curlystack{a=x \\ b \in
         \Ne(y)}$,
        \item if the update is $U3: (a - b) \rightsquigarrow (a \quad b)$ then
        $\curlystack{b=y \\ a \in \Ad(x)}$ or $\curlystack{b=x \\ a \in \Ne(y)}$ or
        $\curlystack{a=y \\ b \in \Ad(x)}$ or $\curlystack{a=x \\ b \in \Ad(y)}$,
        \item if the update is $U4: (a - b) \rightsquigarrow (a \rightarrow b)$ then
        $y=b$ or $\curlystack{a=y \\ b \in \Ad(x)}$ or $\curlystack{b = y \\ a \in
        \Ad(x)}$.
        \item if the update is $U5: (a \rightarrow b) \rightsquigarrow (a \quad b)$ then
        $y=b$ or $\curlystack{a=x \\ b \in \Ne(y)}$ or $\curlystack{b = x \\ a \in
        \Ne(y)}$.
        \item if the update is $U6: (a \rightarrow b) \rightsquigarrow (a - b)$ then
        $y=b$ or $\curlystack{a=y \\ b \in \Ad(x)}$ or $\curlystack{b = y \\ a \in
        \Ad(x)}$.
        \item if the update is $U7: (a \rightarrow b) \rightsquigarrow (a \leftarrow b)$
        then $y\in\{a,b\}$.
    \end{itemize}
\end{opup}

\begin{proof}
    According to \Cref{lemma:I5}, either $\Pa(y)$ changed, or $[\Ne(y) \cap \Ad(x)]$
    changed.

    We now study the necessary conditions for each update, if it was applied and made
    $\mathbf{I5}(x,y, T; E)$ become true.
    \begin{itemize}
        \item U1 does not change $\Pa(y)$, and cannot remove any element from any
        $\Ne(y')$ or $\Ad(x')$. So by \Cref{lemma:I5}, (ii.c) or (ii.d) happened. Hence:
        $\curlystack{b=y \\ a \in \Ad(x)}$ or $\curlystack{b=x \\ a \in \Ne(y)}$ or
        $\curlystack{a=y \\ b \in \Ad(x)}$ or $\curlystack{a=x \\ b \in \Ne(y)}$.
        \item U2 changes $\Pa(y)$ if $y = b$. It can also add an element to $\Ad^u(x)$
        so: $y=b$ or $\curlystack{b=x \\ a \in \Ne(y)}$ or $\curlystack{a=x \\ b \in
        \Ne(y)}$.
        \item U3 does not change $\Pa(y)$, and does not add any element to any $\Ne(y')$
        or $\Ad(x')$.\\
        So by \Cref{lemma:I5}, (ii.a) or (ii.b) happened. Hence: $\curlystack{a=x \\ b
        \in \Ne(y) }$ or $\curlystack{a=y \\ b \in \Ad(x)}$ or $\curlystack{b=x \\ a \in
        \Ne(y) }$ or $\curlystack{b=y \\ a \in \Ad(x)}$.
        \item U4 changes $\Pa(y)$ if $y = b$. It also removes an element from $\Ne^u(a)$
        and $\Ne^u(b)$.\\
        So: $y=b$ or $\curlystack{a=y \\ b \in \Ad(x)}$ or $\curlystack{b = y \\ a \in
        \Ad(x)}$.
        \item U5 changes $\Pa(y)$ if $y = b$. It also removes an element from $\Ad^u(a)$
        and $\Ad^u(b)$.\\
        So: $y=b$ or $\curlystack{a=x \\ b \in \Ne(y)}$ or $\curlystack{b = x \\ a \in
        \Ne(y)}$.
        \item U6 changes $\Pa(y)$ if $y = b$. It also adds an element to $\Ne^u(a)$ and
        $\Ne^u(b)$.\\
        So: $y=b$ or $\curlystack{a=y \\ b \in \Ad(x)}$ or $\curlystack{b = y \\ a \in
        \Ad(x)}$.
        \item U7 changes $\Pa(y)$ if $y \in \{a,b\}$. It does not change any $\Ne(y')$
        or $\Ad(x')$. So $y \in \{a,b\}$.
    \end{itemize}
\end{proof}

\subsubsection{D1}

\begin{opup}[Updates on D1]
    Assume $\mathbf{D1}(x,y, T; E)$ is false and becomes true after an update involving
    $(a,b)$. Then, 
    \begin{itemize}
        \item if the update is $U1: (a \quad b) \rightsquigarrow (a - b)$ then $\{a,b\}
        = \{x,y\}$,
        \item if the update is $U2: (a \quad b) \rightsquigarrow (a \rightarrow b)$ then
        $(a,b) = (x,y)$,
        \item the update cannot be $U3: (a - b) \rightsquigarrow (a \quad b)$,
        \item the update cannot be $U4: (a - b) \rightsquigarrow (a \rightarrow b)$,
        \item the update cannot be $U5: (a \rightarrow b) \rightsquigarrow (a \quad b)$,
        \item if the update is $U6: (a \rightarrow b) \rightsquigarrow (a - b)$ then
        $(a,b) = (y,x)$,
        \item if the update is $U7: (a \rightarrow b) \rightsquigarrow (a \leftarrow b)$
        then $(a,b) = (y,x)$.
    \end{itemize}
\end{opup}

\begin{proof}
    We recall that $\mathbf{D1}(x,y, C; E)$ is: $y \in \Ch(x) \cup \Ne(x)$. Assume 
    $\mathbf{D1}(x,y, C; E)$ is false and becomes true after an update involving $(a,b)$.
    \begin{itemize}
        \item U3 and U5 cannot add any element to $y$'s children or neighbors. So they
        cannot make $\mathbf{D1}(x,y, C; E)$ become true.
        \item U4 does not add any elements to any $\Ch^u(x') \cup \Ne^u(x')$, so it cannot
        make $\mathbf{D1}(x,y, C; E)$ become true.
        \item U1 only adds $a$ to $\Ne^u(b)\cup \Ch^u(b)$ and $b$ to $\Ne^u(a)\cup \Ch^u(a)$. So $\{a,b\} = \{x,y\}$.
        \item U2 only adds $b$ to $\Ch^u(a)\cup \Ne^u(a)$ so $(a,b) = (x,y)$.
        \item U6 only adds $a$ to $\Ch^u(b) \cup \Ne^u(b)$ so $(a,b) = (y,x)$.
        \item U7 only adds $a$ to $\Ch^u(b)\cup \Ne^u(b)$ so $(a,b) = (y,x)$. 
    \end{itemize}
\end{proof}

\subsubsection{D2}
We start with a general lemma for D2.

\begin{lemma}[D2 to become true]\label{lemma:D2} Assume $\mathbf{D2}(x,y, C; E)$ is
    false and becomes true after an edge update about $(a,b)$. Then $\Ne^u(y)$ gained an
    element that was in $\Ad(x)$, or $\Ad^u(x)$ gained an element that was in $\Ne(y)$.
\end{lemma}

\begin{proof}
    We recall that $\mathbf{D2}(x,y, C; E)$ is: $C \subset \Ne(y) \cap \Ad(x)$. Assume
    $\mathbf{D2}(x,y, C; E)$ is false and becomes true after an edge update about
    $(a,b)$. Since $C$ does not change, we there exists $c\in C$ such that $c \not\in
    \Ne(y) \cap \Ad(x)$ and $c \in \Ne^u(y) \cap \Ad^u(x)$. Then we conclude with a similar 
    reasoning as in \Cref{lemma:I5}.
\end{proof}

\begin{opup}[Updates on D2]
    Assume $\mathbf{D2}(x,y, C; E)$ is false and becomes true after an update involving
    $(a,b)$. Then, 
    \begin{itemize}
        \item if the update is $U1: (a \quad b) \rightsquigarrow (a - b)$ then 
        $\curlystack{b=y \\ a \in \Ad(x)}$ or $\curlystack{b=x \\ a \in \Ne(y)}$ or
        $\curlystack{a=y \\ b \in \Ad(x)}$ or $\curlystack{a=x \\ b \in \Ne(y)}$,
        \item if the update is $U2: (a \quad b) \rightsquigarrow (a \rightarrow b)$ then
        $\curlystack{b=x \\ a \in \Ne(y)}$ or $\curlystack{a=x \\ b \in \Ne(y)}$,
        \item the update cannot be $U3: (a - b) \rightsquigarrow (a \quad b)$,
        \item the update cannot be $U4: (a - b) \rightsquigarrow (a \rightarrow b)$,
        \item the update cannot be $U5: (a \rightarrow b) \rightsquigarrow (a \quad b)$,
        \item if the update is $U6: (a \rightarrow b) \rightsquigarrow (a - b)$ then
        $\curlystack{a=y \\ b \in \Ad(x)}$ or $\curlystack{b=y \\ a \in \Ad(x)}$,
        \item the update cannot be $U7: (a \rightarrow b) \rightsquigarrow (a \leftarrow b)$.
    \end{itemize}
\end{opup}

\begin{proof}
    According to \Cref{lemma:D2}, either $\Ne^u(y)$ gained an element that was in
    $\Ad(x)$, or $\Ad^u(x)$ gained an element that was in $\Ne(y)$. We now study the
    necessary conditions for each update, if it was applied and made $\mathbf{D2}(x,y, C;
    E)$ become true.
    \begin{itemize}
        \item U3, U4, U5 and U7 do not add any element to $\Ne^u(y)$ or $\Ad^u(x)$. So they
        cannot make $\mathbf{D2}(x,y, C; E)$ become true.
        \item U1's only modifications to sets $\Ne(y')$ and $\Ad(x')$ are to add $a$ to $\Ne^u(b)$, add $a$ to $\Ad^u(b)$, and 
        same exchanging $a$ and $b$. So by \Cref{lemma:D2}, $\curlystack{b=y \\ a \in \Ad(x)}$ or
        $\curlystack{b=x \\ a \in \Ne(y)}$ or $\curlystack{a=y \\ b \in \Ad(x)}$ or $\curlystack{a=x \\ b \in \Ne(y)}$.
        \item U2 's only modifications to sets $\Ne(y')$ and $\Ad(x')$ are to add $b$ to
        and $\Ad^u(a)$ and same exchanging $a$ and $b$. So by \Cref{lemma:D2} 
        $\curlystack{a=x \\ b \in \Ne(y)}$ or $\curlystack{b = x \\ a \in \Ne(y)}$.
        \item U6's only modifications to sets $\Ne(y')$ and $\Ad(x')$ are to add $a$ to
        $\Ne^u(b)$ and same exchanging $a$ and $b$. So by \Cref{lemma:D2} 
        $\curlystack{b=y \\ a \in \Ad(x)}$ or $\curlystack{a=y \\ b \in \Ad(x)}$. 
    \end{itemize}
\end{proof}

\subsubsection{D3}

\begin{opup}[Updates on D3]
    Assume $\mathbf{D3}(x,y, C; E)$ is false and becomes true after an edge update involving
    $(a,b)$. Also assume that $\mathbf{D2}(x,y, C; E)$ holds true after the edge update, then
    \begin{itemize}
        \item if the update is $U1: (a \quad b) \rightsquigarrow (a - b)$ then $\{a,b\}
        \subset \Ne^u(y) \cap \Ad^u(x)$,
        \item if the update is $U2: (a \quad b) \rightsquigarrow (a \rightarrow b)$ then
        $\{a,b\} \subset \Ne^u(y) \cap \Ad^u(x)$,
        \item the update cannot be $U3: (a - b) \rightsquigarrow (a \quad b)$,
        \item the update cannot be $U4: (a - b) \rightsquigarrow (a \rightarrow b)$,
        \item the update cannot be $U5: (a \rightarrow b) \rightsquigarrow (a \quad b)$,
        \item the update cannot be $U6: (a \rightarrow b) \rightsquigarrow (a - b)$,
        \item the update cannot be $U7: (a \rightarrow b) \rightsquigarrow (a \leftarrow b)$.
    \end{itemize}
\end{opup}

\begin{proof}
    We recall that $\mathbf{D3}(x,y, C; E)$ is: $C$ is a clique. Assume 
    $\mathbf{D3}(x,y, C; E)$ is false and becomes true after an update involving $(a,b)$.
    Similarly to
\Cref{lemma:I3}, since $C$ is not changed by an edge update, the only way for 
$\mathbf{D3}(x,y, C; E)$ to become true is for the edge update to connect two nodes in $C$ that
were not adjacent before. Only U1 and U2 render two nodes adjacent, namely $a$ and $b$,
so they are the only updates that can make $\mathbf{D3}(x,y, C; E)$ become true.

For U1 and U2, we have: $\{a,b\} \subset C$. If we further assume that $\mathbf{D2}(x,y, C; E)$ holds true after the edge update, then
$\{a,b\} \subset C \subset \Ne^u(y) \cap \Ad^u(x)$.
\end{proof}

\subsubsection{D4}
\begin{opup}[Updates on D4]
    Assume $\mathbf{D4}(x,y, C; E)$ is false and becomes true after an edge update involving
    $(a,b)$. Then, 
    \begin{itemize}
        \item the update cannot be $U1: (a \quad b) \rightsquigarrow (a - b)$,
        \item if the update is $U2: (a \quad b) \rightsquigarrow (a \rightarrow b)$ then
        $y = b$,
        \item the update cannot be $U3: (a - b) \rightsquigarrow (a \quad b)$,
        \item if the update is $U4: (a - b) \rightsquigarrow (a \rightarrow b)$, then
        $y=b$,
        \item if the update is $U5: (a \rightarrow b) \rightsquigarrow (a \quad b)$, then $y=b$,
        \item if the update is $U6: (a \rightarrow b) \rightsquigarrow (a - b)$ then $y=b$,
        \item if the update is $U7: (a \rightarrow b) \rightsquigarrow (a \leftarrow b)$ then
        $y\in\{a,b\}$.
    \end{itemize}

    \begin{proof}
        Recall that $\mathbf{D4}(x,y, C; E)$ is: $E = C \cup \Pa(y)$. Since $C$ and $E$
        do not change, the only way for $\mathbf{D4}(x,y, C; E)$ to become true is for
        the edge update to change $\Pa(y)$. The only updates that can change $\Pa(y)$
        are U2, U4, U5, U6, and U7, when $y$ is $b$, or U7 when $y$ is $a$ or $b$.
    \end{proof}
    
\end{opup}

\subsubsection{R1 to R6}
The reverse operators are very similar to the insert operators, and the necessary conditions
can be adapted. We add them to \Cref{tab:operator_updates}.

\subsection{Issue with Fast GES }
\label{appendix:subsec:fast_ges}
We found an issue with the Fast GES algorithm that may explain its degraded performance
compared to the GES algorithm. 

During its efficient update of the operators, fGES computes the score of all possible
operators Insert($x,y,T$) for a pair of nodes $x$ and $y$, but only saves the Insert
with the highest score. However, this insert might not be a valid operator (for example,
it might not satisfy the I3 constraint). Meanwhile, another Insert($x,y,T'$) for the
same pair of nodes $x$ and $y$ might be valid and have the highest score of all valid
operators. Such an operator would be missed by fGES.

\end{document}